\documentclass{article}

\PassOptionsToPackage{numbers, compress}{natbib}

\usepackage{titletoc}
\newcommand\DoToC{%
  \startcontents
  \printcontents{}{1}{\hrulefill\vskip0pt}
  \vskip0pt \noindent\hrulefill
  }
\usepackage[preprint]{neurips_2024}

\usepackage[utf8]{inputenc} 
\usepackage[T1]{fontenc}    
\usepackage{hyperref}       
\usepackage{url}            
\usepackage{booktabs}       
\usepackage{amsfonts}       
\usepackage{nicefrac}       
\usepackage{microtype}      
\usepackage{xcolor}         

\usepackage{multirow} 
\usepackage{subcaption}
\usepackage{amsmath}
\usepackage{amssymb}
\usepackage{amsthm}
\usepackage{mathtools}
\usepackage{enumitem}
\usepackage{wrapfig}
\usepackage{lmodern}
\usepackage{bold-extra}

\theoremstyle{plain}

\newtheorem{definition}{Definition}
\newtheorem{corollary}{Corollary}
\newtheorem{lemma}{Lemma}
\newtheorem{proposition}{Proposition}

\newtheorem{theorem}{Theorem}

\newcommand{\method}{\textbf{CtrlNS}}

\title{Causal Temporal Representation Learning with Nonstationary Sparse Transition}
\author{
    Xiangchen Song \\
    Carnegie Mellon University\\
    \texttt{xiangchensong@cmu.edu} \\
}
\author{%
    Xiangchen Song$^{1}$ \quad
    Zijian Li$^{2}$ \quad
    Guangyi Chen$^{1,2}$ \quad
    Yujia Zheng$^{1}$ \\
    \textbf{Yewen Fan}$^{1}$ \quad
    \textbf{Xinshuai Dong}$^{1}$ \quad
    \textbf{Kun Zhang}$^{1,2}$\\
    $^1$Carnegie Mellon University\\
    $^2$Mohamed bin Zayed University of Artificial Intelligence
}

\begin{document}

\maketitle

\begin{abstract}

    Causal Temporal Representation Learning (Ctrl) methods aim to identify the temporal causal dynamics of complex nonstationary temporal sequences. Despite the success of existing Ctrl methods, they require either directly observing the domain variables or assuming a Markov prior on them. Such requirements limit the application of these methods in real-world scenarios when we do not have such prior knowledge of the domain variables. To address this problem, this work adopts a sparse transition assumption, aligned with intuitive human understanding, and presents identifiability results from a theoretical perspective. In particular, we explore under what conditions on the significance of the variability of the transitions we can build a model to identify the distribution shifts. Based on the theoretical result, we introduce a novel framework, \emph{Causal Temporal Representation Learning with Nonstationary Sparse Transition} (\method), designed to leverage the constraints on transition sparsity and conditional independence to reliably identify both distribution shifts and latent factors. Our experimental evaluations on synthetic and real-world datasets demonstrate significant improvements over existing baselines, highlighting the effectiveness of our approach.

\end{abstract}

\section{Introduction}

Causal learning from sequential data remains a fundamental yet challenging task~\cite{berzuini2012causality,ghysels2016testing,friston2009causal}. 
Discovering temporal causal relations among \emph{observed} variables has been extensively studied in the literature~\cite{granger1980testing,gong2015discovering,hyvarinen2010estimation}.
However, in many real-world scenarios such as video understanding~\cite{behrmann2021long}, observed data are generated by causally related latent temporal processes or confounders rather than direct causal edges.
This leads to the task of \emph{causal temporal representation learning}~(Ctrl), which aims to build compact representations that concisely capture the data generation processes by inverting the mixing function that transforms latent factors into observations and identifying the transitions that govern the underlying latent causal dynamics. This learning problem is known to be challenging without specific assumptions~\cite{locatello2019challenging,hyvarinen1999nonlinear}.
The task becomes significantly more complex with \emph{nonstationary} transitions, which are often characterized by multiple distribution shifts across different domains, particularly when these domains or shifts are also unobserved.

Recent advances in unsupervised representation learning, particularly through nonlinear Independent Component Analysis (ICA), have shown promising results in identifying latent variables by incorporating side information such as class labels and domain indices~\cite{hyvarinen2016unsupervised,hyvarinen2017nonlinear,hyvarinen2019nonlinear,khemakhem2020variational,sorrenson2020disentanglement,halva2020hidden,pmlr-v162-kong22a,lachapelle2022disentanglement,lachapelle2024nonparametric,zheng2023generalizing}. For time-series data, historical information is widely utilized to enhance the identifiability of latent temporal causal processes~\cite{halva2021disentangling,klindt2020towards,yao2021learning,yao2022temporally}. However, existing studies primarily derive results under stationary conditions~\cite{hyvarinen2017nonlinear,klindt2020towards} or nonstationary conditions with observed domain indices~\cite{khemakhem2020variational,yao2021learning,yao2022temporally}.
These methods are limited in application as general time series data are typically nonstationary and domain information is difficult to obtain. Recent studies~\cite{halva2020hidden, balsells2023identifiability, song2023temporally, li2024how} have adopted a Markov structure to handle nonstationary domain variables and can infer domain indices directly from observed data. (More related work can be found in Appendix~\ref{ap:related work}.) However, these methods face significant limitations; some are inadequate for modeling time-delayed causal relationships in latent spaces, and they rely on the Markov property, which cannot adequately capture the arbitrary nonstationary variations in domain variables. This leads us to the following important yet unresolved question:

\begin{center}
	{ \textit{How can we establish identifiability of nonstationary nonlinear ICA for} } \\
	{ \textit{general sequence data without prior knowledge of domain variables?} }
\end{center}
Relying on observing domain variables or known Markov priors to capture nonstationarity seems counter-intuitive, especially considering how easily humans can identify domain shifts given sufficient variation on transitions, such as video action segmentation~\cite{ridley2022transformers, xu2024efficient} and recognition~\cite{ibrahim2016hierarchical, gavrilyuk2020actor, kim2022detector} tasks.
In this work, we theoretically investigate the conditions on the significance of transition variability to identify distribution shifts.
The core idea is transition clustering, assuming transitions within the same domain are similar, while transitions across different domains are distinct.
Building on this identification theorem, we propose \emph{Causal Temporal Representation Learning with Nonstationary Sparse Transition}~(\method), to identify both distribution shifts and latent temporal dynamics. Specifically, we constrain the complexity of the transition function to identify domain shifts. Subsequently, with the identified domain variables, we learn the latent variables using conditional independence constraints. These two processes are jointly optimized within a VAE framework.

The main contributions of this work are as follows: (1) To our best knowledge, this is the first identifiability result that handles nonstationary time-delayed causally related latent temporal processes without prior knowledge of the domain variables. (2) We present \method, a principled framework for recovering both nonstationary domain variables and time-delayed latent causal dynamics. (3) Experiments on synthetic and real-world datasets demonstrate the effectiveness of the proposed method in recovering latent variables and domain indices.

\section{Problem Formulation}\label{sec:Problem Formulation}
\subsection{Nonstationary Time Series Generative Model}
\begin{wrapfigure}{r}{0.3\textwidth}
    \centering
    \vspace{-5.5em}
    \includegraphics[width=\linewidth]{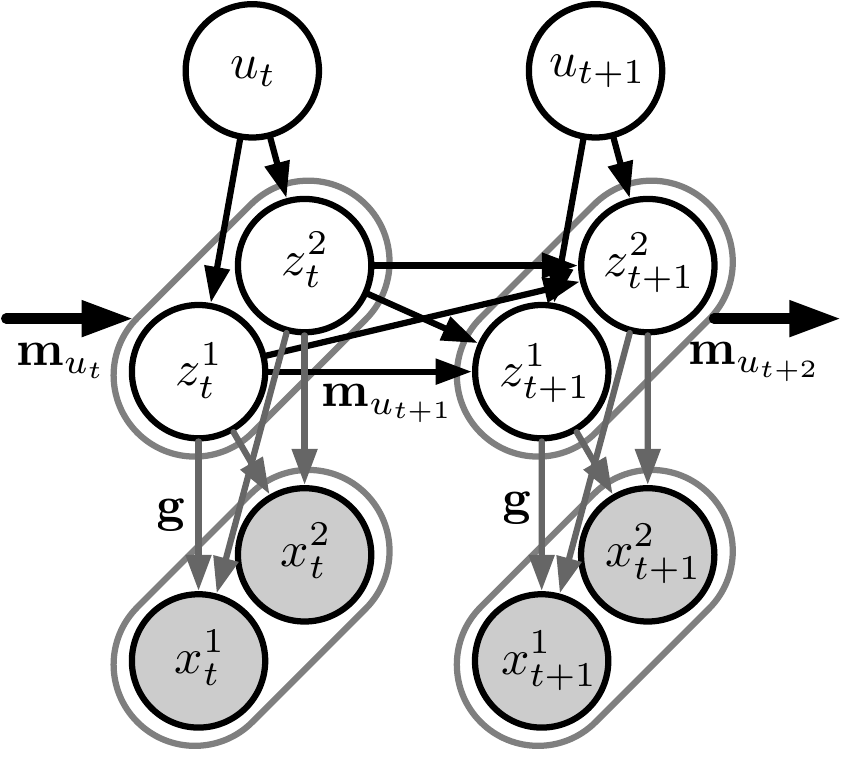}
    \caption{Graphical model for nonstationary causally related time-delayed time-series data with unobserved domain variables \(u_t\).}
    \label{fig:pgm}
\end{wrapfigure}

We first introduce a nonstationary time-series generative model. Our  observational dataset is \(\mathcal{D} = \{\mathbf{x}_t\}_{t=1}^{T}\), where \(\mathbf{x}_t\in \mathbb{R}^{n}\) is produced from causally related, time-delayed latent components \(\mathbf{z}_t \in \mathbb{R}^{n}\) through an invertible mixing function \(\mathbf{g}\):
\begin{equation}
    \mathbf{x}_t = \mathbf{g}(\mathbf{z}_t).
    \label{Eq:mixing function}
\end{equation}
In the nonstationary setting, transitions within the latent space vary over time. Define \(u\) as the domain or regime index variable, with \(u_t\) corresponding to time step \(t\). Assuming \(U\) distinct regimes, i.e., \(u_t \in \{1, 2, \ldots, U\}\), each regime exhibits unknown distribution shifts. Those regimes are characterized by \(U\) different transition functions \(\{\mathbf{m}_u\}_{u=1}^{U}\), which were originally explored in~\cite{huang2022adarl} through change factors to capture these distribution shifts in transition dynamics.
The \(i\)-th component of \(\mathbf{z}_t\), is then generated via \(i\)-th component of \(\mathbf{m}\):
\begin{equation}
    z_{t,i} = m_i\left(u_t, \{z_{t',j} \mid z_{t',j} \in \mathbf{Pa}(z_{t,i})\}, \epsilon_{t,i}\right),
    \label{Eq:transition function of z}
\end{equation}
where \(\mathbf{Pa}(z_{t,i})\) represents the set of latent factors directly influencing \(z_{t,i}\), which may include any subset of \(\mathbf{z}_{<t} = \left\{z_{\tau,i} \mid \tau \in \{1,2,\dots,t-1\}, i\in \{1,2,\dots,n\}\right\}\).
For analytical simplicity, we assume that the parents in the causal graph are restricted to elements in \(\mathbf{z}_{t-1}\). Extensions to higher-order cases, which involve multistep, time-delayed causal relations, are discussed in Appendix S1.5 of \cite{yao2022temporally}. These extensions are orthogonal to our contributions and are therefore omitted here for brevity.
Importantly, in a nonstationary context, \(\mathbf{Pa}(\cdot)\) may also be sensitive to the domain index \(u_t\), indicating that causal dependency graphs vary across different domains or regimes, which will be revisited in our discussion on identifiability. We assume that the generation process for each \(i\)-th component of \(\mathbf{z}_t\) is mutually independent, conditioned on \(\mathbf{z}_{<t}\) and \(u_t\), and we assume that the noise terms \(\epsilon_{t,i}\) are mutually independent across dimensions and over time, thereby excluding instantaneous causal interactions among latent causal factors. Figure~\ref{fig:pgm} illustrates the graphical model for this setting.

\subsection{Identifiability of Domain Variables and Latent Causal Processes}

We introduce the identifiability for both domain variables and time-delayed latent causal processes in Definitions \ref{def: Identifiable Domain Variables} and \ref{def: Identifiable Latent Causal Processes}, respectively. If the estimated latent processes are identifiable at least up to a permutation and component-wise invertible transformations, then the latent causal relationships are also immediately identifiable. This is because conditional independence relations comprehensively characterize the time-delayed causal relations within a time-delayed causally sufficient system, in which there are no latent causal confounders in the causal processes. Notably, invertible component-wise transformations on latent causal processes preserve their conditional independence relationships. We now present definitions related to observational equivalence, the identifiability of domain variables and latent causal processes.

\begin{definition}[Observational Equivalence]\label{def: Observational Equivalence}
Formally, consider \(\{\mathbf{x}_t\}_{t=1}^{T}\) as a sequence of observed variables generated by true temporally causal latent processes specified by \((\mathbf{m}, \mathbf{u}, p({\boldsymbol{\epsilon}}), \mathbf{g})\) given in Eqs.~\eqref{Eq:mixing function} and \eqref{Eq:transition function of z}. Here, \(\mathbf{m}\) and \(\boldsymbol{\epsilon}\) denote the concatenated vector form across \(n\) dimensions in the latent space. Similarly \(\mathbf{u}\) for timestep \(1\) to \(T\). A learned generative model \((\hat{\mathbf{m}}, \hat{\mathbf{u}}, \hat{p}({\boldsymbol{\epsilon}}), \hat{\mathbf{g}})\) is observationally equivalent to the ground truth one \((\mathbf{m}, \mathbf{u}, p({\boldsymbol{\epsilon}}), \mathbf{g})\) if the model distribution \(p_{\hat{\mathbf{m}}, \hat{\mathbf{u}}, \hat{p}_{\boldsymbol{\epsilon}}, \hat{\mathbf{g}}}(\{\mathbf{x}_t\}_{t=1}^{T})\) matches the data distribution \(p_{\mathbf{m}, \mathbf{u}, p_{\boldsymbol{\epsilon}}, \mathbf{g}}(\{\mathbf{x}_t\}_{t=1}^{T})\) everywhere.
\end{definition}

\begin{definition}[Identifiable Domain Variables]\label{def: Identifiable Domain Variables}
    Domain variables are said to be identifiable up to label swapping if observational equivalence (Def. \ref{def: Observational Equivalence}) implies identifiability of domain variables up to label swapping or a permutation \(\sigma\) for domain indices:
    \begin{equation}\label{eq:iden_domain_var}
    p_{\hat{\mathbf{m}}, \hat{\mathbf{u}}, \hat{p}_{\boldsymbol{\epsilon}}, \hat{\mathbf{g}}}(\{\mathbf{x}_t\}_{t=1}^{T}) = p_{\mathbf{m}, \mathbf{u}, p_{\boldsymbol{\epsilon}}, \mathbf{g}}(\{\mathbf{x}_t\}_{t=1}^{T})
    \Rightarrow \hat{u}_t = \sigma(u_t), \forall t \in \{1,2,\ldots,T\}.
    \end{equation}
\end{definition}
\begin{definition}[Identifiable Latent Causal Processes]\label{def: Identifiable Latent Causal Processes}
The latent causal processes are said to be identifiable if observational equivalence (Def. \ref{def: Observational Equivalence}) leads to the identifiability of latent variables up to a permutation \(\pi\) and component-wise invertible transformation \(\mathcal{T}\):
\begin{equation}\label{eq:iden_latent_causal}
p_{\hat{\mathbf{m}}, \hat{\mathbf{u}}, \hat{p}_{\boldsymbol{\epsilon}}, \hat{\mathbf{g}}}(\{\mathbf{x}_t\}_{t=1}^{T}) = p_{\mathbf{m}, \mathbf{u}, p_{\boldsymbol{\epsilon}}, \mathbf{g}}(\{\mathbf{x}_t\}_{t=1}^{T})
\Rightarrow \hat{\mathbf{g}}^{-1}(\mathbf{x}_t) = \mathcal{T} \circ \pi \circ \mathbf{g}^{-1}(\mathbf{x}_t), \quad \forall \mathbf{x}_t \in \mathcal{X},
\end{equation}
where \(\mathcal{X}\) denotes the observation space.
\end{definition}

\section{Identifiability Theory}\label{sec:Identifiability Theory}

In this section, we demonstrate that under mild conditions, the domain variables \(u_t\) are identifiable up to label swapping and the latent variables \(\mathbf{z}_t\) are identifiable up to permutation and component-wise transformations. We partition our theoretical discussion into two sections: (1) identifiability of nonstationary discrete domain variables \(u_t\) and (2) identifiability of latent causal processes.
We slightly extend the usage of \(\operatorname{supp}(\cdot)\) to define the square matrix support and the support of a square matrix function as follows:
\begin{definition}[Matrix Support]\label{def:supp_matrix}
The support (set) of a square matrix \(\mathbf{A} \in \mathbb{R}^{n \times n}\) is defined using the indices of non-zero entries as:
\begin{equation}
    \operatorname{supp}(\mathbf{A}) \coloneqq \left\{(i,j) \mid \mathbf{A}_{i,j} \neq 0\right\}.
\end{equation}
\end{definition}

\begin{definition}[Matrix Function Support]\label{def:supp_matrix_function}
The support (set) of a square matrix function \(\mathbf{A} : \Theta \rightarrow \mathbb{R}^{n \times n}\) is defined as:
\begin{equation}
    \operatorname{supp}(\mathbf{A}(\Theta)) \coloneqq \left\{(i,j) \mid \exists \theta \in \Theta, \mathbf{A}(\theta)_{i,j} \neq 0 \right\}.
\end{equation}
\end{definition}

For brevity, let \(\mathcal{M}\) and \(\widehat{\mathcal{M}}\) denote the \(n \times n\) binary matrices reflecting the support of the Jacobian \(\mathbf{J}_{\mathbf{m}}(\mathbf{z}_t)\) and \(\mathbf{J}_{\hat{\mathbf{m}}}(\hat{\mathbf{z}}_t)\) respectively. The \((i,j)\) entry of \(\mathcal{M}\) is \(1\) if and only if \((i,j) \in \operatorname{supp}(\mathbf{J}_{\mathbf{m}})\). And we can further define the transition complexity using its Fr\'{e}chet norm, \(|\mathcal{M}| = \sum_{i,j} \mathcal{M}_{i,j}\), and similarly for \(\widehat{\mathcal{M}}\). In the nonstationary context, this support matrix is a function of the domain index \(u\), denoted as \(\mathcal{M}_u\) and \(\widehat{\mathcal{M}}_u\). Additionally, we introduce the concept of weakly diverse lossy transitions for the data generation process:
\begin{definition}[Weakly Diverse Lossy Transition]\label{def:Diversely Lossy Transformation}
The set of transition functions described in Eq.~\eqref{Eq:transition function of z} is said to be diverse lossy if it satisfies the following conditions:
\begin{enumerate}[itemsep=-2pt]
        \item (\underline{Lossy}) For every time and indices tuple \((t,i,j)\) with edge \(z_{t-1,i} \to z_{t,j}\) representing a causal link defined with the parents set \(\mathbf{Pa}(z_{t,j})\) in Eq.~\ref{Eq:transition function of z}, transition function \(m_j\) is a lossy transformation w.r.t. \(z_{t-1,i}\) i.e., there exists an open set \(S_{t,i,j}\), changing \(z_{t-1,i}\) within this set will not change the value of \(m_j\), i.e. \(\forall z_{t-1,i} \in S_{t,i,j}, \, \frac{\partial m_j}{\partial z_{t-1,i}} = 0\).
        \item (\underline{Weakly Diverse}) For every element \(z_{t-1,i}\) of the latent variable \(\mathbf{z}_{t-1}\) and its corresponding children set \(\mathcal{J}_{t,i} = \{j \mid z_{t-1,i} \in \mathbf{Pa}(z_{t,j}), j\in \{1,2,\dots,n\}\}\), transition functions \(\{m_j\}_{j \in \mathcal{J}_{t,i}}\) are  weakly diverse i.e., the intersection of the sets \(S_{t,i} = \cap_{j\in \mathcal{J}_{t,i}} S_{t,i,j}\) is not empty, and such sets are diverse, i.e., \(S_{t,i}\neq \emptyset\), and \(S_{t,i,j} \setminus S_{t,i} \neq \emptyset, \forall j\in \mathcal{J}_{t,i}\).
    \end{enumerate}
\end{definition}

\subsection{Identifiability of Domain Variables}

\begin{theorem}[Identifiability of Domain Variables]
\label{thm: identifiability of C}
Suppose that the dataset \(\mathcal{D}\) are generated from the nonstationary data generation process as described in Eqs. \eqref{Eq:mixing function} and \eqref{Eq:transition function of z}. Suppose the transitions are weakly diverse lossy~(Def.~\ref{def:Diversely Lossy Transformation})
and the following assumptions hold:
\begin{enumerate}[label=\roman*.,ref=\roman*]

    \item \label{as: Mechanism Separability}\underline{(Mechanism Separability)} There exists a ground truth mapping \(\mathcal{C}: \mathcal{X} \times \mathcal{X}  \to \mathcal{U}\) determined the real domain indices, i.e., \(u_t = \mathcal{C}(\mathbf{x}_{t-1}, \mathbf{x}_{t})\).
    \item \label{as: Mechanism Sparsity} \underline{(Mechanism Sparsity)} The estimated transition complexity on dataset \(\mathcal{D}\) is less than or equal to ground truth transition complexity, i.e., \(\mathbb{E}_{\mathcal{D}} | \widehat{\mathcal{M}}_{\hat{u}} | \leq \mathbb{E}_{\mathcal{D}}  | \mathcal{M}_{u} |\).\label{eq: Mechanism Sparsity}
    
    \item \label{as: Mechanism Variability}\underline{(Mechanism Variability)} Mechanisms are sufficiently different. For all $u\neq u'$, $\mathcal{M}_{u} \neq \mathcal{M}_{u'}$ i.e. there exists index $(i,j)$ such that \(\left[\mathcal{M}_{u}\right]_{i,j} \neq \left[\mathcal{M}_{u'}\right]_{i,j}\).
    
\end{enumerate}
Then the domain variables $u_t$ is identifiable up to label swapping (Def.~\ref{def: Identifiable Domain Variables}).
\end{theorem}

Theorem~\ref{thm: identifiability of C} states that if we successfully learn a set of estimated transitions \(\{\hat{\mathbf{m}}_u\}_{u=1}^{U}\), the decoder \(\hat{\mathbf{g}}\), and the domain clustering assignment \(\hat{\mathcal{C}}\), where \(\hat{\mathbf{m}}_u\) corresponds to the estimation of Eq.~\eqref{Eq:transition function of z} for a particular regime or domain \(u\), and the system can fit the data as follows:
\small\begin{equation}\label{eq:sparse_transition}
    \hat{\mathbf{x}}_t = \hat{\mathbf{g}} \circ \hat{\mathbf{m}}_{\hat{u}_t} \circ \hat{\mathbf{g}}^{-1}(\mathbf{x}_{t-1}) \quad\text{and}\quad \hat{u}_t = \hat{\mathcal{C}}(\mathbf{x}_{t-1}, \mathbf{x}_t),
\end{equation}\normalsize
assuming that the transition complexity is kept low (as per Assumption~\ref{as: Mechanism Sparsity}). Then the estimated domain variables \(\hat{u}_t\) must be the true domain variables \(u_t\) up to a permutation. 

\textbf{Proof sketch}
The core idea of this proof is to demonstrate that the global minimum of transition complexity can only be achieved when the domain variables \(u_t\) are correctly estimated. (1) First, when we have an optimal decoder estimation \(\hat{\mathbf{g}}^*\) which is a component-wise transformation of the ground truth, incorrect estimations of \(u_t\) will strictly increase the transition complexity, i.e., \(\mathbb{E}_{\mathcal{D}} |\widehat{\mathcal{M}}^{*}_{\hat{u}}| > \mathbb{E}_{\mathcal{D}} |\widehat{\mathcal{M}}^{*}_{u}|\). (2) Second, with arbitrary estimations \(\hat{u}_t\), the transition complexity for non-optimal decoder estimation \(\hat{\mathbf{g}}\) will be equal to or higher than that for the optimal \(\hat{\mathbf{g}}^*\), i.e., \(\mathbb{E}_{\mathcal{D}} |\widehat{\mathcal{M}}_{\hat{u}}| \geq \mathbb{E}_{\mathcal{D}} |\widehat{\mathcal{M}}^{*}_{\hat{u}}|\). Thus, the global minimum of transition complexity can only be achieved when \(u_t\) is optimally estimated, which must be a permuted version of the ground truth domain variables \(u_t\). For a detailed proof, see appendix~\ref{ap: Appendix Proof for Theorem identifiability of C}.

\subsection{Remark on Mechanism Variability}\label{sec:remark}

The assumption of mechanism variability, as outlined in Assumption~\ref{as: Mechanism Variability}, requires the Jacobian support matrices across domains must be distinct, which means that the causal graph linking past states (\(\mathbf{z}_{t-1}\)) to current states (\(\mathbf{z}_t\)) varies by at least one edge. But addressing scenarios where the causal graphs remain constant but the functions behind the edges change is challenging without additional assumptions; more detailed discussion on why this is in general challenging can be found in the Appendix~\ref{ap:Discussion on the hardness of identifiability under mechanism function variability}.

To effectively address these scenarios, we extend the concept of the Jacobian support matrix by incorporating higher-order derivatives. This extension provides a more detailed characterization of the variability in transition functions across different domains. We now present the following definition to formalize this concept:

\begin{definition}[Higher Order Partial Derivative Support Matrix]\label{def:Higher Order Partial Derivative Matrix Support}
The \(k\)-th order partial derivative support matrix for transition \(\mathbf{m}\) denoted as \(\mathcal{M}^k\) is a binary \(n\times n\) matrix with
\small\begin{equation}
    \left[\mathcal{M}^k\right]_{i,j} = 1 \iff \exists \mathbf{z} \in \mathcal{Z}, \frac{\partial^k m_j}{\partial z_i^k} \neq 0.
\end{equation}\normalsize
\end{definition}
We utilize the variability in the higher-order partial derivative support matrix to extend the identifiability results of Theorem~\ref{thm: identifiability of C}. This extension applies to cases where the causal graphs remain identical across two domains, yet the transition functions take different forms.

\begin{corollary}[Identifiability under Function Variability]\label{cor: identifiability of domain variables under mechanism function variability}
    Suppose the data \(\mathcal{D}\) is generated from the nonstationary data generation process described in \eqref{Eq:mixing function} and \eqref{Eq:transition function of z}. Assume the transitions are weakly diverse lossy (Def.~\ref{def:Diversely Lossy Transformation}), and the mechanism separability assumption~\ref{as: Mechanism Separability} along with the following assumptions hold:
    \begin{enumerate}[label=\roman*.,ref=\roman*, start=5]

    \item \underline{(Mechanism Function Variability)} Mechanism Functions are sufficiently different. There exists \(K \in \mathbb{N}\) such that for all $u\neq u'$, there exists \(k \leq K\), $\mathcal{M}_{u}^k \neq \mathcal{M}_{u'}^k$ i.e. there exists index $(i,j)$ such that \(\left[\mathcal{M}_{u}^k\right]_{i,j} \neq \left[\mathcal{M}_{u'}^k\right]_{i,j}\).\label{as: Mechanism Function Variability}

    \item \underline{(Higher Order Mechanism Sparsity)} The estimated transition complexity on dataset \(\mathcal{D}\) is no more than ground truth transition complexity,\label{as: Higher Order Mechanism Sparsity}
    \small\begin{equation}
    \mathbb{E}_{\mathcal{D}} \sum_{k=1}^K | \widehat{\mathcal{M}}_{\hat{u}}^k | \leq \mathbb{E}_{\mathcal{D}}  \sum_{k=1}^K | \mathcal{M}_{u}^k |. 
    \end{equation}\normalsize
    \end{enumerate}
    Then the domain variables $u_t$ are identifiable up to label swapping (Def.~\ref{def: Identifiable Domain Variables}).
\end{corollary}
We utilize the fact that for two distinct domains, there exists an edge in the causal graph, and its \(k\)-th order partial derivative supports are different, making the two domains separable. The detailed proof of this extension is provided in Appendix~\ref{ap:Proof of Corollary}.
\subsection{Identifiability of Latent Causal Process}

Once the identifiability of \(u_t\) is achieved, the problem reduces to a nonstationary temporal nonlinear ICA with an observed domain index.

Leveraging the sufficient variability approach proposed in \cite{yao2022temporally}, we demonstrate full identifiability. This sufficient variability concept is further incorporated into the following lemma, adapted from Theorem 2 in \cite{yao2022temporally}:

\begin{lemma}[Theorem 2 in Yao et al., \cite{yao2022temporally}]\label{lemma:tdrl identifiability_of_z}
Suppose that the data \(\mathcal{D}\) are generated from the nonstationary data generation process as described in Eqs. \eqref{Eq:mixing function} and \eqref{Eq:transition function of z}. Let \(\eta_{kt}(u)\) denote the logarithmic density of \(k\)-th variable in \(\mathbf{z}_t\), i.e., \(\eta_{kt}(u)\triangleq \log p(z_{t,k} | \mathbf{z}_{t-1}, u)\), and there exists an invertible function $\hat{\mathbf{g}}$  that maps $\mathbf{x}_t$ to $\hat{\mathbf{z}}_t$, i.e., \(\hat{\mathbf{z}}_t = \hat{\mathbf{g}}(\mathbf{x}_t)\)

such that the components of $\hat{\mathbf{z}}_t$ are mutually  independent conditional on $\hat{\mathbf{z}}_{t-1}$.
\underline{(Sufficient variability)} Let
\small
\begin{align} \label{Eq:v1}
    \mathbf{v}_{k,t}(u) &\triangleq \Big(\frac{\partial^2 \eta_{kt}(u)}{\partial z_{t,k} \partial z_{t-1,1}}, \frac{\partial^2 \eta_{kt}(u)}{\partial z_{t,k} \partial z_{t-1,2}}, ..., \frac{\partial^2 \eta_{kt}(u)}{\partial z_{t,k} \partial z_{t-1,n}} \Big)^\intercal,\\ 
    \mathring{\mathbf{v}}_{k,t}(u) &\triangleq \Big(\frac{\partial^3 \eta_{kt}(u)}{\partial z_{t,k}^2 \partial z_{t-1,1}}, \frac{\partial^3 \eta_{kt}(u)}{\partial z_{t,k}^2 \partial z_{t-1,2}}, ..., \frac{\partial^3 \eta_{kt}(u)}{\partial z_{t,k}^2 \partial z_{t-1,n}} \Big)^\intercal. 
\end{align}

\begin{align} \label{Eq:v2}
    \mathbf{s}_{kt} &\triangleq \Big( \mathbf{v}_{kt}(1)^\intercal, ..., 
    \mathbf{v}_{kt}(U)^\intercal, 
    \frac{\partial^2 \eta_{kt}(2)}{\partial z_{t,k}^2 } - 
    \frac{\partial^2 \eta_{kt}(1)}{\partial z_{t,k}^2 }, ...,
    \frac{\partial^2 \eta_{kt}(U)}{\partial z_{t,k}^2 } - 
    \frac{\partial^2 \eta_{kt}(U-1)}{\partial z_{t,k}^2 }
    \Big)^\intercal,\\
    \mathring{\mathbf{s}}_{kt} &\triangleq \Big( \mathring{\mathbf{v}}_{kt}(1)^\intercal, ..., 
    \mathring{\mathbf{v}}_{kt}(U)^\intercal, 
    \frac{\partial \eta_{kt}(2)}{\partial z_{t,k} } - 
    \frac{\partial \eta_{kt}(1)}{\partial z_{t,k} }, ...,
    \frac{\partial \eta_{kt}(U)}{\partial z_{t,k} } - 
    \frac{\partial \eta_{kt}(U-1)}{\partial z_{t,k} }
    \Big)^\intercal.
\end{align}\normalsize
Suppose $\mathbf{x}_t = \mathbf{g}(\mathbf{z}_t)$ and that the conditional distribution $p(z_{k,t} \,|\, \mathbf{z}_{t-1})$ may change across $m$ domains. Suppose that the components of $\mathbf{z}_t$ are mutually independent conditional on $\mathbf{z}_{t-1}$ in each context. Assume that the components of $\hat{\mathbf{z}}_t$ produced by \(\hat{\mathbf{g}}\) are also mutually independent conditional on $\hat{\mathbf{z}}_{t-1}$. 
If the $2n$ function vectors $\mathbf{s}_{k,t}$ and $\mathring{\mathbf{s}}_{k,t}$, with $k=1,2,...,n$, are linearly independent, then $\hat{\mathbf{z}}_t$ is a permuted invertible component-wise transformation of $\mathbf{z}_t$. 
\end{lemma}
Then, in conjunction with Theorem~\ref{thm: identifiability of C}, complete identifiability is achieved for both the domain variables \(u_t\) and the independent components \(\mathbf{z}_t\). See detailed proof in Appendix~\ref{ap:Proof of Identifiability of the Latent Causal Processes}.
\begin{theorem}[Identifiability of the Latent Causal Processes]\label{thm: identifiability of z}
    Suppose that the data \(\mathcal{D}\) are generated from the nonstationary data generation process as described in Eqs.~\eqref{Eq:mixing function} and \eqref{Eq:transition function of z}, which satisfies the conditions in both Theorem~\ref{thm: identifiability of C} and Lemma~\ref{lemma:tdrl identifiability_of_z}, then the domain variables \(u_t\) are identifiable up to label swapping~(Def.~\ref{def: Identifiable Domain Variables}) and latent causal process $\mathbf{z}_t$ are identifiable up to permutation and a component-wise transformation~(Def.~\ref{def: Identifiable Latent Causal Processes}).
\end{theorem}
\textbf{Discussion on Assumptions}

The proof of Theorem~\ref{thm: identifiability of C} relies on several key assumptions which align with human intuition for understanding of domain transitions.
Firstly, \emph{separability} states that if human observers cannot distinguish between two domains, it is unlikely that automated systems can achieve this distinction either.
Secondly, \emph{variability} requires that the transitions across domains are significant enough to be noticeable by humans, implying that there must be at least one altered edge in the causal graph across the domains.

The mechanism \emph{sparsity} is a standard assumption that has been previously explored in \citep{zheng2022on,zheng2023generalizing,lachapelle2024nonparametric} using sparsity regularization to enforce the sparsity of the estimated function.
The assumption of \emph{weakly diverse lossy transitions} is a mild and realistic condition in real-world scenarios, allowing for identical future latent states with differing past states. 

The \emph{sufficient variability} in Theorem~\ref{thm: identifiability of z} is widely explored and adopted in nonlinear ICA literature~\cite{hyvarinen2019nonlinear,yao2021learning,yao2022temporally,song2023temporally,li2024how}. For a more detailed discussion of the feasibility and intuition behind these assumptions, we refer the reader to the Appendix~\ref{ap: Discussion on Assumptions}.

\section{The \method~Framework}\label{sec:NSCtrl}

\subsection{Model Architecture}
\begin{wrapfigure}{r}{0.4\textwidth}
    \vspace{-2.5em}
    \centering
    \includegraphics[width=\linewidth]{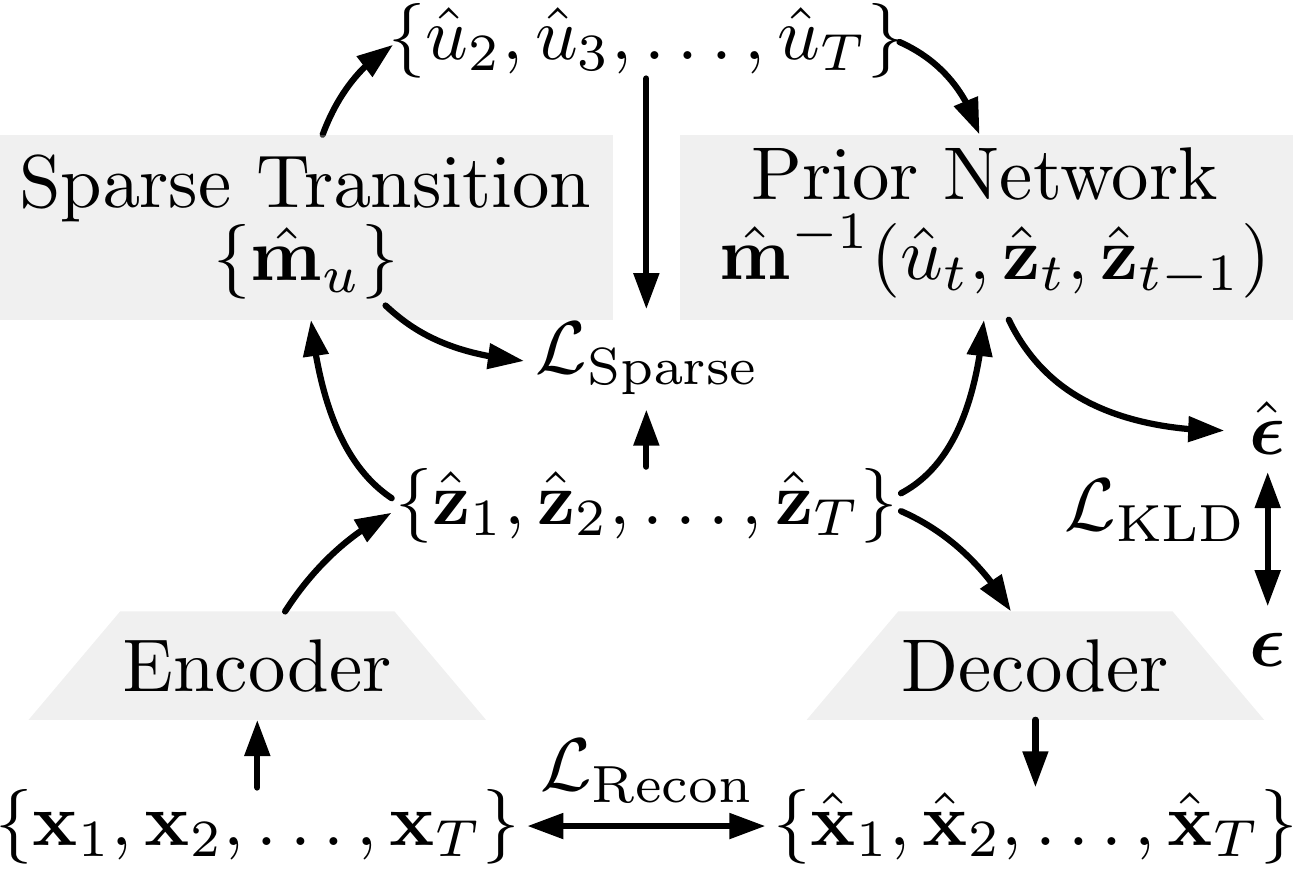}
    \caption{Illustration of \method~with (1) Sparse Transition, (2) Prior Network, (3) Encoder-Decoder Module.}\label{fig:arch}
    \vspace{-2em}
\end{wrapfigure}

Our framework builds on VAE~\cite{kingma2013auto,yingzhen2018disentangled} architecture, incorporating dedicate modules to handle nonstationarity. It enforces the conditions discussed in Sec.~\ref{sec:Identifiability Theory} as constraints. As shown in Fig.~\ref{fig:arch}, the framework consists of three primary components: (1) Sparse Transition, (2) Prior Network, and (3) Encoder-Decoder.

\textbf{Sparse Transition}

The transition module in our framework is designed to estimate transition functions \(\{\hat{\mathbf{m}}_u\}_{u=1}^{U}\) and a clustering function \(\hat{\mathcal{C}}\) as specified in Eq.~\eqref{eq:sparse_transition}. As highlighted in Sec.~\ref{sec:Identifiability Theory}, the primary objective of this module is to model the transitions in the latent space and minimize the empirical transition complexity. To achieve this, we implemented \(U\) different transition networks for various \(\hat{\mathbf{m}}(\hat{u}_t,\cdot)\) and added sparsity regularization to the transition functions via a sparsity loss. A gating function with a (hard)-Gumbel-Softmax function was used to generate \(\hat{u}_t\), which was then employed to select the corresponding transition network \(\hat{\mathbf{m}}_{\hat{u}_t}\). This network was further used to calculate the transition loss, which is explained in detail in Sec.~\ref{sec:Optimization}.

\textbf{Prior Network}
The Prior Network module aims to effectively estimate the prior distribution \(p(\hat{z}_{t,i} \,|\,\hat{\mathbf{z}}_{t-1}, \hat{u}_t)\). This is achieved by evaluating \(p(\hat{z}_t \,|\,\hat{\mathbf{z}}_{t-1}, \hat{u}_t) = p_{\epsilon_i}\left(\hat{m}_{i}^{-1}(\hat{u}_t,\hat{z}_{t,i}, \hat{\mathbf{z}}_{t-1})\right)\Big|\frac{\partial \hat{m}_{i}^{-1}}{\partial \hat{z}_{t,i}}\Big|\), where \(\hat{m}_{i}^{-1}(\hat{u}_t,\cdot)\) is the learned holistic inverse dynamics model.
To ensure the conditional independence of the estimated latent variables, \(p(\hat{\mathbf{z}}_t \,|\, \hat{\mathbf{z}}_{t-1})\), we utilize an isomorphic noise distribution for \(\epsilon\) and aggregate all estimated component densities to obtain the joint distribution \(p(\hat{\mathbf{z}}_t \,|\, \hat{\mathbf{z}}_{t-1}, \hat{u}_t)\) as shown in Eq.~\eqref{eq:np-trans}. Given the lower-triangular nature of the Jacobian, its determinant can be computed as the product of its diagonal terms. Detailed derivations is provided in Appendix~\ref{ap:derive}.
\small\begin{equation}\label{eq:np-trans}
\log p\left(\hat{\mathbf{z}}_t \mid \hat{\mathbf{z}}_{t-1}, \hat{u}_t\right) = \underbrace{\sum_{i=1}^n \log p(\hat{\epsilon}_i \mid \hat{u}_t)}_{\text{Conditional independence}} + \underbrace{\sum_{i=1}^n \log \Big| \frac{\partial \hat{m}^{-1}_i}{\partial \hat{z}_{t,i}}\Big|}_{\text{Lower-triangular Jacobian}}
\end{equation}\normalsize
\textbf{Encoder-Decoder}
The third component is an Encoder-Decoder module that utilizes reconstruction loss to enforce the invertibility of the learned mixing function $\hat{\mathbf{g}}$. Specifically, the encoder fits the demixing function $\hat{\mathbf{g}}^{-1}$ and the decoder fits the mixing function $\hat{\mathbf{g}}$.

\subsection{Optimization}\label{sec:Optimization}
The first training objective of \method~is to fit the estimated transitions with minimum transition complexity according to Eq.~\eqref{eq:sparse_transition}:
\small\begin{equation}
    \mathcal{L}_{\text{sparse}} \triangleq \underbrace{\mathbb{E}_{\mathcal{D}} L(\hat{\mathbf{m}}_{\hat{u}_t}(\hat{\mathbf{z}}_{t-1}), \hat{\mathbf{z}}_t)}_{\text{Transition loss}} + \underbrace{\mathbb{E}_{\mathcal{D}} | \widehat{\mathcal{M}}_{\hat{u}} |}_{\text{Sparsity loss}},
\end{equation}\normalsize
where \(L(\cdot,\cdot)\) is a regression loss function to fit the transition estimations, and the sparsity loss is approximated via \(L_2\) norm of the parameter in the transition estimation functions.

Then the second part is to maximize the Evidence Lower BOund (\textsc{ELBO}) for the VAE framework, which can be written as follows~(complete derivation steps are in Appendix~\ref{ap:ELBO}):

\small\begin{equation}
\begin{split}
\textsc{ELBO}\triangleq \mathbb{E}_{\mathbf{z}_t}\underbrace{\sum_{t=1}^{T} \log p_{\text{data}}(\mathbf{x}_t\mid\mathbf{z}_t)}_{-\mathcal{L}_{\text{Recon}}} 
+ \underbrace{\sum_{t=1}^{T}\log  p_{\text{data}}(\mathbf{z}_t\mid \mathbf{z}_{t-1},u_t) 
- \sum_{t=1}^{T}\log q_{\phi}(\mathbf{z}_t\mid\mathbf{x}_t)}_{-\mathcal{L}_{\text{KLD}}}
\end{split}
\end{equation}\normalsize
We use mean-squared error for the reconstruction likelihood loss $\mathcal{L}_{\text{Recon}}$.
The KL divergence  $\mathcal{L}_{\text{KLD}}$ is estimated via a sampling approach since with a learned nonparametric transition prior, the distribution does not have an explicit form.
Specifically, we obtain the log-likelihood of the posterior, evaluate the prior $\log p\left(\hat{\mathbf{z}}_t \mid \hat{\mathbf{z}}_{t-1}, \hat{u}_t\right)$ in Eq.~\eqref{eq:np-trans}, and compute their mean difference in the dataset as the KL loss: $\mathcal{L}_{\text{KLD}} = \mathbb{E}_{\mathbf{\hat z}_t \sim q\left(\mathbf{\hat z}_t \mid \mathbf{x}_t\right)} \log q(\mathbf{\hat z}_t|\mathbf{x}_t) - \log p\left(\hat{\mathbf{z}}_t \mid \hat{\mathbf{z}}_{t-1}, \hat{u}_t\right)$.

\section{Experiments}\label{sec:Experiments}
We assessed the identifiability performance of \method\ on both synthetic and real-world datasets.
For synthetic datasets, where we control the data generation process completely, we conducted a comprehensive evaluation. This evaluation covers the full spectrum of unknown nonstationary causal temporal representation learning, including metrics for both domain variables and the latent causal processes.
In real-world scenarios, \method\ was employed in video action segmentation tasks. The evaluation metrics focus on the accuracy of action estimation for each video frame, which directly reflects the identifiability of domain variables.

\subsection{Synthetic Experiments on Causal Representation Learning}
\textbf{Experiment Setup} \ \ 
For domain variables, we assessed the \emph{clustering accuracy} (\textbf{Acc}) to estimate discrete domain variables \(u_t\). As the label order in clustering algorithms is not predetermined, we selected the order or permutation that yielded the highest accuracy score.
For the latent causal processes, we computed the \emph{mean correlation coefficient} (\textbf{MCC}) between the estimated latent variables \(\hat{\mathbf{z}}_t\) and the ground truth \(\mathbf{z}_t\). The MCC, a standard measure in the ICA literature for continuous variables, assesses the identifiability of the learned latent causal processes. We adjusted the reported MCC values in Table~\ref{tab: synthetic experiment} by multiplying them by 100 to enhance the significance of the comparisons.

\textbf{Baselines} \ \
We compared our method with identifiable nonlinear ICA methods: 
(1) BetaVAE \cite{higgins2017betavae}, which ignores both history and nonstationarity information. 
(2) i-VAE \cite{khemakhem2020variational} and TCL \cite{hyvarinen2016unsupervised}, which leverage nonstationarity to establish identifiability but assume independent factors. 
(3) SlowVAE \cite{klindt2020towards} and PCL \cite{hyvarinen2017nonlinear}, which exploit temporal constraints but assume independent sources and stationary processes. 
(4) TDRL \cite{yao2022temporally}, which assumes nonstationary causal processes but with observed domain indices. 
(5) HMNLICA \cite{halva2020hidden}, which considers the unobserved nonstationary part in the data generation process but does not allow any causally related time-delayed relations. 
(6) NCTRL \cite{song2023temporally}, which extends HMNLICA to an autoregressive setting to allow causally related time-delayed relations in the latent space but still assumes a Markov chain on the domain variables.

\textbf{Result and Analysis} \ \
We generate synthetic datasets that satisfy our identifiability conditions in Theorems~\ref{thm: identifiability of C} and \ref{thm: identifiability of z}, detailed procedures are in Appendix~\ref{ap:synthetic}. The primary findings are presented in Table~\ref{tab: synthetic experiment}. 
Note: the MCC metric is consistently available in all methods; however, the Acc metric for \(u_t\) is only applicable to methods capable of estimating domain variables \(u_t\).
In the first row of Table~\ref{tab: synthetic experiment}, we evaluated a recent nonlinear temporal\begin{wraptable}{r}{0.7\textwidth}
    \caption{Experiment results of synthetic dataset on baseline models and the proposed \method. All experiments were conducted using three different random seeds to calculate the average and standard deviation. The best results are highlighted in \textbf{bold}.}
    \label{tab: synthetic experiment}
    \centering
    \renewcommand{\arraystretch}{0.95}
    \setlength{\tabcolsep}{4pt}
    \begin{tabular}{cccc}
    \toprule
    \(u_t\) & \textbf{Method} & \(\mathbf{z}_t\) \textbf{MCC} & \(u_t\) \textbf{Acc} (\%) \\
    \midrule
    Ground Truth &TDRL(GT) & 96.93 $\pm$ 0.16 & -\\
    \midrule
    \multirow{6}{*}{N/A} &TCL   & 24.19 $\pm$ 0.85  &\multirow{6}{*}{-} \\
    &PCL   & 38.46 $\pm$ 6.85 &\\
    &BetaVAE & 42.37 $\pm$ 1.47 &\\
    &SlowVAE & 41.82 $\pm$ 2.55 &\\
    &i-VAE  & 81.60 $\pm$ 2.51 &\\
    &TDRL  & 53.45 $\pm$ 1.31 &\\
    \midrule
    \multirow{3}{*}{Estimated}&HMNLICA & 17.82 $\pm$ 30.87    & 13.67 $\pm$ 23.67\\
    &NCTRL  & 47.27 $\pm$ 2.15   & 34.94 $\pm$ 4.20\\
    &\textbf{\method} & \textbf{96.74 $\pm$ 0.17} & \textbf{98.21 $\pm$ 0.05}\\
    \bottomrule
    \end{tabular}
    \vspace{-1em}
\end{wraptable}ICA method, TDRL, providing ground truth \(u_t\) to establish an upper performance limit for the proposed framework.
The high MCC (> 0.95) indicates the model's identifiability.
Subsequently, the table lists six baseline methods that neglect the nonstationary domain variables, with none achieving a high MCC. The remaining approaches, including our proposed \method, are able to estimate the domain variables \(u_t\) and recover the latent variables. In particular, HMNLICA exhibits instability during training, leading to considerable performance variability. This instability stems from HMNLICA’s inability to allow time-delayed causal relationships among hidden variables \(\mathbf{z}_t\), leading to model training failure when the actual domain variables deviate from the Markov assumption. In contrast, NCTRL, which extends TDRL under the same assumption, demonstrates enhanced stability and performance over HMNLICA by accommodating transitions in \(\mathbf{z}_t\). However, since they use incorrect assumption on the nonstationary domain variables, the performance of those methods can be even worse than methods which do not include the domain information. Nevertheless, considering the significant nonstationarity and deviation from the Markov properties, those methods struggled to robustly estimate either the domain variables or the latent causal processes. Compared to all baselines, our proposed \method~reliably recovers both \(u_t\) (MCC > 0.95) and \(\mathbf{z}_t\) (Acc > 95\%), and the MCC is on par with the upper performance bound when domain variables are given, justifying it effectivess.

\begin{wrapfigure}{l}{0.4\textwidth}
    \centering
    \vspace{-1.3em}
    \includegraphics[width=\linewidth]{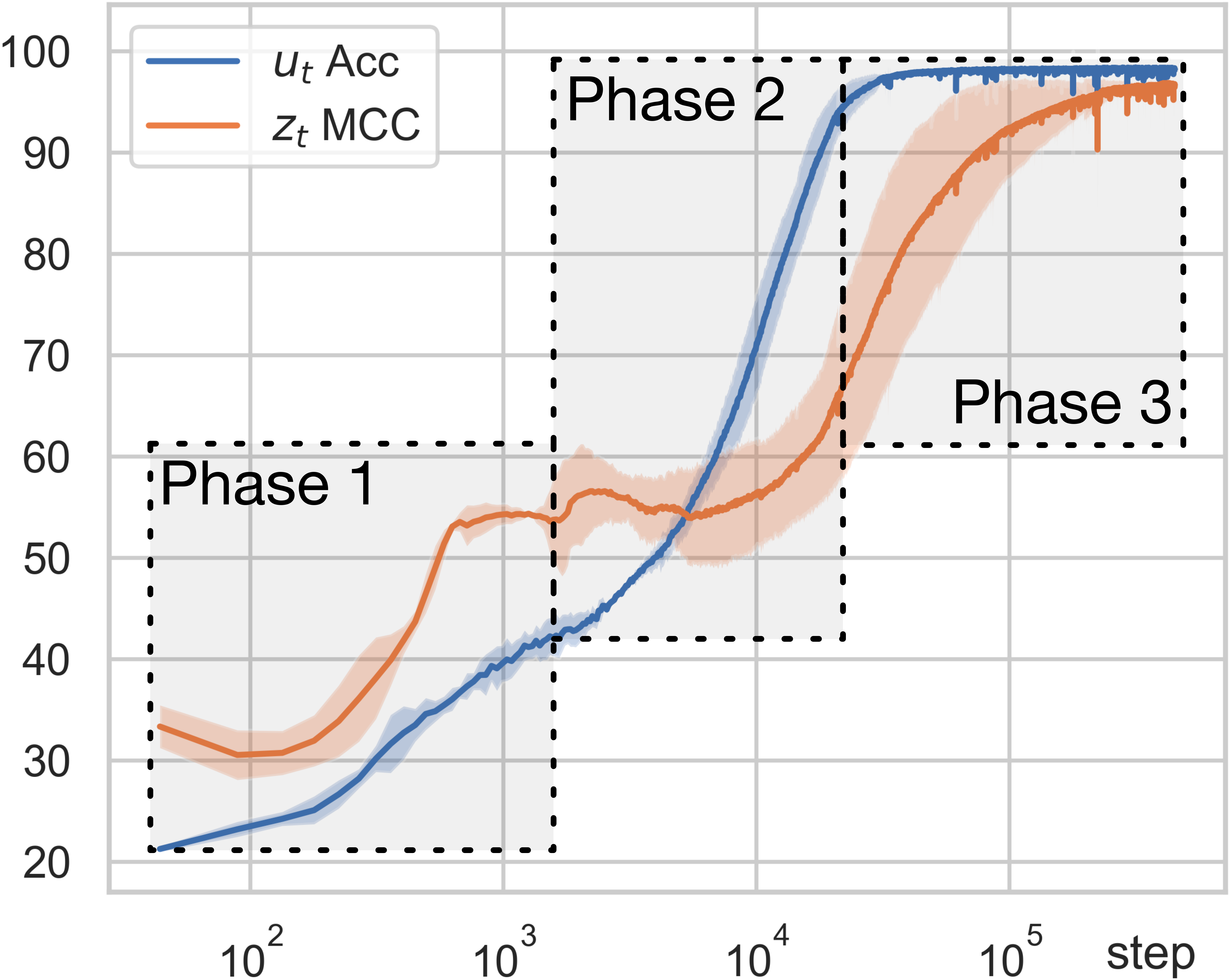}
    \vspace{-2mm}
    \caption{Visualization of three phase training process of \method.}\label{fig:simulation three phases}
    \vspace{-1em}
\end{wrapfigure}
\textbf{Detailed Training Analysis} \ \
To further validate our theoretical analysis, we present a visualization of the entire training process for \method~in Figure~\ref{fig:simulation three phases}. It consists of three phases: (1) In Phase 1, the initial estimations for both \(u_t\) and \(\mathbf{z}_t\) are imprecise. (2) During Phase 2, the accuracy of the estimation of \(u_t\) continues to improve, although the quality of the estimation of \(\mathbf{z}_t\) remains relatively unchanged compared with Phase 1. (3) In Phase 3, as \(u_t\) becomes clearly identifiable, the MCC of \(\mathbf{z}_t\) progressively improves, ultimately achieving full identifiability. This three-phase process aligns perfectly with our theoretical predictions. According to Theorem~\ref{thm: identifiability of C}, phases 1 and 2 should exhibit suboptimal \(\mathbf{z}_t\) estimations, while sparsity constraints can still guide training and improve the accuracy for domain variables \(u_t\). Once the accuracy of \(u_t\) approaches high, Theorem~\ref{thm: identifiability of z} drives the improvement in MCC for \(\mathbf{z}_t\) estimations, leading to the final achievement of full identifiability of both latent causal processes for \(\mathbf{z}_t\) and domain variables \(u_t\).

\subsection{Real-world Application on Weakly Supervised Action Segmentation}
\textbf{Experiment Setup}
Our method was tested on the video action segmentation task to estimate actions (domain variables \(u_t\)). Following~\cite{mucon2021,xu2024efficient}, we use the same weakly supervised setting utilizing meta-information, such as action order. The evaluation included several metrics: \emph{Mean-over-Frames} (\textbf{MoF}), the percentage of correctly predicted labels per frame; 
\emph{Intersection-over-Union} (\textbf{IoU}), defined as \(|I \cap I^*|/|I \cup I^*|\); and \emph{Intersection-over-Detection} (\textbf{IoD}), \(|I \cap I^*| / |I|\), \(I^*\) and \(I\) are the ground-truth segment and the predicted segment with the same class. Our evaluation used two datasets: \textbf{Hollywood Extended}~\cite{bojanowski2014weakly}, which includes 937 videos with 16 daily action categories, and \textbf{CrossTask}~\cite{zhukov2019cross}, focusing on 14 of 18 primary tasks related to cooking~\cite{Lu_2021_ICCV}, comprising 2552 videos across 80 action categories. 

\textbf{Model Design} \ \
Our model is build on top of ATBA~\cite{xu2024efficient} method which uses multi-layer transformers as backbone networks. We add our sparse transition module with the sparsity loss function detailed in Sec.~\ref{sec:Optimization}. Specifically, we integrated a temporally latent transition layer into ATBA's backbone, using a transformer layer across time axis for the Hollywood dataset and an LSTM for the CrossTask dataset. To enforce sparsity in the latent transitions, \(L_2\) regularization is applied to the weights of the temporally latent transition layer.

\begin{wraptable}{l}{0.62\textwidth}
\vspace{-1.3em}
\renewcommand{\arraystretch}{0.9}
\setlength{\tabcolsep}{2pt}
\caption{Real-world experiment result on action segmentation task. We use the reported value for the baseline methods from \cite{xu2024efficient}. Best results are highlighted in \textbf{bold}.}\label{tab:main real world}
\centering
\begin{tabular}{ccccc}
\toprule
\textbf{Dataset}            & \textbf{Method}                               & \textbf{MoF}  & \textbf{IoU}   & \textbf{IoD}   \\ 
\midrule
\multirow{6}{*}{Hollywood}  & HMM+RNN~\cite{richard2017weakly}              & -             & 11.9           & -              \\
                            & CDFL~\cite{li2019weakly}                      & 45.0          & 19.5           & 25.8           \\
                            
                            & TASL~\cite{Lu_2021_ICCV}           & 42.1          & 23.3           & 33 \\
                            & MuCon~\cite{mucon2021}                        & -             & 13.9           & -              \\
                            & ATBA~\cite{xu2024efficient}                   & 47.7          & 28.5           & 44.9           \\
                            \cmidrule(l){2-5} 
                            & \method~(\textbf{Ours})                       & \textbf{52.9}{\tiny\(\pm\)3.1} & \textbf{32.7}{\tiny\(\pm\)1.3}  & \textbf{52.4}{\tiny\(\pm\)1.8}  \\
\midrule
\multirow{6}{*}{CrossTask}  & NN-Viterbi~\cite{richard2018neuralnetwork}    & 26.5          & 10.7           & 24.0             \\
                            & CDFL~\cite{li2019weakly}                      & 31.9          & 11.5           & 23.8           \\
                            & TASL~\cite{Lu_2021_ICCV}                      & 40.7          & 14.5           & \textbf{25.1}           \\
                            & POC~\cite{lu2022set}                          & 42.8          & 15.6           & -              \\
                            & ATBA~\cite{xu2024efficient}                   & 50.6          & \textbf{15.7}           & 24.6  \\
                            \cmidrule(l){2-5} 
                            & \method~(\textbf{Ours})                       & \textbf{54.0}{\tiny\(\pm\)0.9} & \textbf{15.7}{\tiny\(\pm\)0.5}  & 23.6{\tiny\(\pm\)0.8}           \\
\bottomrule
\end{tabular}
\vspace{-1em}
\end{wraptable}

\textbf{Result and Analysis} \,
The primary outcomes for real-world applications in action segmentation are summarized in Table~\ref{tab:main real world}. Traditional methods based on hidden Markov models, such as HMM+RNN~\cite{richard2017weakly} and NN-Viterbi~\cite{richard2018neuralnetwork}, face challenges in these real-world scenarios. This observation corroborates our previous discussions on the limitations of earlier identifiability methods~\cite{halva2020hidden,balsells2023identifiability,song2023temporally}, which depend on the Markov assumption for domain variables. Our approach significantly outperforms the baselines in both the Hollywood and CrossTask datasets across most metrics. Especially in the Hollywood dataset, our method outperforms the base ATBA model by quite a large margin. Notably, the Mean-over-Frames~(\textbf{MoF}) metric aligns well with our identifiability results for domain variables \(u_t\). Our method demonstrates substantial superiority in this metric. For Intersection-over-Union (\textbf{IoU}) and Intersection-over-Detection (\textbf{IoD}), our results are comparable to those of the baseline methods in the CrossTask dataset and show its superiority in the Hollywood dataset. Furthermore, our proposed sparse transition module aligns with human intuition and is easily integrated into existing methods, further enhancing its impact in real-world scenarios. Additional discussion with visualization in Appendix~\ref{ap:vis}.

\begin{wraptable}{r}{0.45\textwidth}
    \vspace{-1.3em}
    \renewcommand{\arraystretch}{0.9}
    \caption{Ablation study on sparse transition module in Hollywood dataset.}
    \centering
    \begin{tabular}{lcccc}
    \toprule
         \textbf{Method}    & \textbf{MoF}  & \textbf{IoU}   & \textbf{IoD}   \\ 
    \midrule
         \method            & \textbf{52.9}          & \textbf{32.7}           & \textbf{52.4} \\
         - Complexity             & 50.5          & 31.5           & 51.5 \\
         - Module               & 47.7          & 28.5           & 44.9 \\
    \bottomrule
    \end{tabular}\label{tab:ablation study}
    \vspace{-1em}
\end{wraptable}
\textbf{Abalation Study} \ \
Furthermore, we conducted an ablation study on the sparse transition module, as detailed in Table~\ref{tab:ablation study}. In this study, ``- Complexity'' refers to the configuration where we retain the latent transition layers but omit the sparse transition complexity regularization term from these layers, and ``- Module'' indicates the removal of the entire sparse transition module, effectively reverting the model to the baseline ATBA model. The comparative results in Table~\ref{tab:ablation study} demonstrate that both the dedicated design of the sparse transition module and the complexity regularization term significantly enhance the performance.

\section{Conclusion}\label{sec:Conclusion}
In this study, we developed a comprehensive identifiability theory tailored for general sequential data influenced by nonstationary causal processes under unspecified distributional changes. We then introduced \method, a principled approach to recover both latent causal variables with their time-delayed causal relations, as well as determining the values of domain variables from observational data without relying on distributional or structural prior knowledge. Our experimental results demonstrate that \method~can reliably estimate the domain indices and recover the latent causal process. And such module can be easily adapted to handle real-world scenarios such as action segmentation task.

\bibliography{ref}

\begin{thebibliography}{70}
\providecommand{\natexlab}[1]{#1}
\providecommand{\url}[1]{\texttt{#1}}
\expandafter\ifx\csname urlstyle\endcsname\relax
  \providecommand{\doi}[1]{doi: #1}\else
  \providecommand{\doi}{doi: \begingroup \urlstyle{rm}\Url}\fi

\bibitem[Berzuini et~al.(2012)Berzuini, Dawid, and Bernardinell]{berzuini2012causality}
Carlo Berzuini, Philip Dawid, and Luisa Bernardinell.
\newblock \emph{Causality: Statistical perspectives and applications}.
\newblock John Wiley \& Sons, 2012.

\bibitem[Ghysels et~al.(2016)Ghysels, Hill, and Motegi]{ghysels2016testing}
Eric Ghysels, Jonathan~B Hill, and Kaiji Motegi.
\newblock Testing for granger causality with mixed frequency data.
\newblock \emph{Journal of Econometrics}, 192\penalty0 (1):\penalty0 207--230, 2016.

\bibitem[Friston(2009)]{friston2009causal}
Karl Friston.
\newblock Causal modelling and brain connectivity in functional magnetic resonance imaging.
\newblock \emph{PLoS biology}, 7\penalty0 (2):\penalty0 e1000033, 2009.

\bibitem[Granger(1980)]{granger1980testing}
Clive~WJ Granger.
\newblock Testing for causality: A personal viewpoint.
\newblock \emph{Journal of Economic Dynamics and control}, 2:\penalty0 329--352, 1980.

\bibitem[Gong et~al.(2015)Gong, Zhang, Schoelkopf, Tao, and Geiger]{gong2015discovering}
Mingming Gong, Kun Zhang, Bernhard Schoelkopf, Dacheng Tao, and Philipp Geiger.
\newblock Discovering temporal causal relations from subsampled data.
\newblock In \emph{International Conference on Machine Learning}, pages 1898--1906. PMLR, 2015.

\bibitem[Hyv{\"a}rinen et~al.(2010)Hyv{\"a}rinen, Zhang, Shimizu, and Hoyer]{hyvarinen2010estimation}
Aapo Hyv{\"a}rinen, Kun Zhang, Shohei Shimizu, and Patrik~O Hoyer.
\newblock Estimation of a structural vector autoregression model using non-gaussianity.
\newblock \emph{Journal of Machine Learning Research}, 11\penalty0 (5), 2010.

\bibitem[Behrmann et~al.(2021)Behrmann, Fayyaz, Gall, and Noroozi]{behrmann2021long}
Nadine Behrmann, Mohsen Fayyaz, Juergen Gall, and Mehdi Noroozi.
\newblock Long short view feature decomposition via contrastive video representation learning.
\newblock In \emph{Proceedings of the IEEE/CVF international conference on computer vision}, pages 9244--9253, 2021.

\bibitem[Locatello et~al.(2019)Locatello, Bauer, Lucic, Raetsch, Gelly, Sch{\"o}lkopf, and Bachem]{locatello2019challenging}
Francesco Locatello, Stefan Bauer, Mario Lucic, Gunnar Raetsch, Sylvain Gelly, Bernhard Sch{\"o}lkopf, and Olivier Bachem.
\newblock Challenging common assumptions in the unsupervised learning of disentangled representations.
\newblock In \emph{international conference on machine learning}, pages 4114--4124. PMLR, 2019.

\bibitem[Hyv{\"a}rinen and Pajunen(1999)]{hyvarinen1999nonlinear}
Aapo Hyv{\"a}rinen and Petteri Pajunen.
\newblock Nonlinear independent component analysis: Existence and uniqueness results.
\newblock \emph{Neural networks}, 12\penalty0 (3):\penalty0 429--439, 1999.

\bibitem[Hyvarinen and Morioka(2016)]{hyvarinen2016unsupervised}
Aapo Hyvarinen and Hiroshi Morioka.
\newblock Unsupervised feature extraction by time-contrastive learning and nonlinear ica.
\newblock \emph{Advances in Neural Information Processing Systems}, 29:\penalty0 3765--3773, 2016.

\bibitem[Hyvarinen and Morioka(2017)]{hyvarinen2017nonlinear}
Aapo Hyvarinen and Hiroshi Morioka.
\newblock Nonlinear ica of temporally dependent stationary sources.
\newblock In \emph{Artificial Intelligence and Statistics}, pages 460--469. PMLR, 2017.

\bibitem[Hyvarinen et~al.(2019)Hyvarinen, Sasaki, and Turner]{hyvarinen2019nonlinear}
Aapo Hyvarinen, Hiroaki Sasaki, and Richard Turner.
\newblock Nonlinear ica using auxiliary variables and generalized contrastive learning.
\newblock In \emph{The 22nd International Conference on Artificial Intelligence and Statistics}, pages 859--868. PMLR, 2019.

\bibitem[Khemakhem et~al.(2020)Khemakhem, Kingma, Monti, and Hyvarinen]{khemakhem2020variational}
Ilyes Khemakhem, Diederik Kingma, Ricardo Monti, and Aapo Hyvarinen.
\newblock Variational autoencoders and nonlinear ica: A unifying framework.
\newblock In \emph{International Conference on Artificial Intelligence and Statistics}, pages 2207--2217. PMLR, 2020.

\bibitem[Sorrenson et~al.(2020)Sorrenson, Rother, and K{\"o}the]{sorrenson2020disentanglement}
Peter Sorrenson, Carsten Rother, and Ullrich K{\"o}the.
\newblock Disentanglement by nonlinear ica with general incompressible-flow networks (gin).
\newblock \emph{arXiv preprint arXiv:2001.04872}, 2020.

\bibitem[H{\"a}lv{\"a} and Hyvarinen(2020)]{halva2020hidden}
Hermanni H{\"a}lv{\"a} and Aapo Hyvarinen.
\newblock Hidden markov nonlinear ica: Unsupervised learning from nonstationary time series.
\newblock In \emph{Conference on Uncertainty in Artificial Intelligence}, pages 939--948. PMLR, 2020.

\bibitem[Kong et~al.(2022)Kong, Xie, Yao, Zheng, Chen, Stojanov, Akinwande, and Zhang]{pmlr-v162-kong22a}
Lingjing Kong, Shaoan Xie, Weiran Yao, Yujia Zheng, Guangyi Chen, Petar Stojanov, Victor Akinwande, and Kun Zhang.
\newblock Partial disentanglement for domain adaptation.
\newblock In Kamalika Chaudhuri, Stefanie Jegelka, Le~Song, Csaba Szepesvari, Gang Niu, and Sivan Sabato, editors, \emph{Proceedings of the 39th International Conference on Machine Learning}, volume 162 of \emph{Proceedings of Machine Learning Research}, pages 11455--11472. PMLR, 17--23 Jul 2022.
\newblock URL \url{https://proceedings.mlr.press/v162/kong22a.html}.

\bibitem[Lachapelle et~al.(2022)Lachapelle, Rodriguez, Sharma, Everett, PRIOL, Lacoste, and Lacoste-Julien]{lachapelle2022disentanglement}
Sebastien Lachapelle, Pau Rodriguez, Yash Sharma, Katie~E Everett, R{\'e}mi~LE PRIOL, Alexandre Lacoste, and Simon Lacoste-Julien.
\newblock Disentanglement via mechanism sparsity regularization: A new principle for nonlinear {ICA}.
\newblock In \emph{First Conference on Causal Learning and Reasoning}, 2022.
\newblock URL \url{https://openreview.net/forum?id=dHsFFekd_-o}.

\bibitem[Lachapelle et~al.(2024)Lachapelle, L{\'o}pez, Sharma, Everett, Priol, Lacoste, and Lacoste-Julien]{lachapelle2024nonparametric}
S{\'e}bastien Lachapelle, Pau~Rodr{\'\i}guez L{\'o}pez, Yash Sharma, Katie Everett, R{\'e}mi~Le Priol, Alexandre Lacoste, and Simon Lacoste-Julien.
\newblock Nonparametric partial disentanglement via mechanism sparsity: Sparse actions, interventions and sparse temporal dependencies.
\newblock \emph{arXiv preprint arXiv:2401.04890}, 2024.

\bibitem[Zheng and Zhang(2023)]{zheng2023generalizing}
Yujia Zheng and Kun Zhang.
\newblock Generalizing nonlinear {ICA} beyond structural sparsity.
\newblock In \emph{Thirty-seventh Conference on Neural Information Processing Systems}, 2023.
\newblock URL \url{https://openreview.net/forum?id=gI1SOgW3kw}.

\bibitem[H{\"a}lv{\"a} et~al.(2021)H{\"a}lv{\"a}, Corff, Leh{\'e}ricy, So, Zhu, Gassiat, and Hyvarinen]{halva2021disentangling}
Hermanni H{\"a}lv{\"a}, Sylvain~Le Corff, Luc Leh{\'e}ricy, Jonathan So, Yongjie Zhu, Elisabeth Gassiat, and Aapo Hyvarinen.
\newblock Disentangling identifiable features from noisy data with structured nonlinear ica.
\newblock \emph{arXiv preprint arXiv:2106.09620}, 2021.

\bibitem[Klindt et~al.(2020)Klindt, Schott, Sharma, Ustyuzhaninov, Brendel, Bethge, and Paiton]{klindt2020towards}
David Klindt, Lukas Schott, Yash Sharma, Ivan Ustyuzhaninov, Wieland Brendel, Matthias Bethge, and Dylan Paiton.
\newblock Towards nonlinear disentanglement in natural data with temporal sparse coding.
\newblock \emph{arXiv preprint arXiv:2007.10930}, 2020.

\bibitem[Yao et~al.(2022{\natexlab{a}})Yao, Sun, Ho, Sun, and Zhang]{yao2021learning}
Weiran Yao, Yuewen Sun, Alex Ho, Changyin Sun, and Kun Zhang.
\newblock Learning temporally causal latent processes from general temporal data.
\newblock In \emph{International Conference on Learning Representations}, 2022{\natexlab{a}}.
\newblock URL \url{https://openreview.net/forum?id=RDlLMjLJXdq}.

\bibitem[Yao et~al.(2022{\natexlab{b}})Yao, Chen, and Zhang]{yao2022temporally}
Weiran Yao, Guangyi Chen, and Kun Zhang.
\newblock Temporally disentangled representation learning.
\newblock \emph{Advances in Neural Information Processing Systems}, 35:\penalty0 26492--26503, 2022{\natexlab{b}}.

\bibitem[Balsells-Rodas et~al.(2023)Balsells-Rodas, Wang, and Li]{balsells2023identifiability}
Carles Balsells-Rodas, Yixin Wang, and Yingzhen Li.
\newblock On the identifiability of switching dynamical systems, 2023.

\bibitem[Song et~al.(2023)Song, Yao, Fan, Dong, Chen, Niebles, Xing, and Zhang]{song2023temporally}
Xiangchen Song, Weiran Yao, Yewen Fan, Xinshuai Dong, Guangyi Chen, Juan~Carlos Niebles, Eric Xing, and Kun Zhang.
\newblock Temporally disentangled representation learning under unknown nonstationarity.
\newblock In \emph{Thirty-seventh Conference on Neural Information Processing Systems}, 2023.
\newblock URL \url{https://openreview.net/forum?id=V8GHCGYLkf}.

\bibitem[Li et~al.(2024)Li, Cai, Yang, Huang, Chen, Shen, Chen, Song, Hao, and Zhang]{li2024how}
Zijian Li, Ruichu Cai, Zhenhui Yang, Haiqin Huang, Guangyi Chen, Yifan Shen, Zhengming Chen, Xiangchen Song, Zhifeng Hao, and Kun Zhang.
\newblock When and how: Learning identifiable latent states for nonstationary time series forecasting, 2024.

\bibitem[Ridley et~al.(2022)Ridley, Coskun, Tan, Navab, and Tombari]{ridley2022transformers}
John Ridley, Huseyin Coskun, David~Joseph Tan, Nassir Navab, and Federico Tombari.
\newblock Transformers in action: weakly supervised action segmentation.
\newblock \emph{arXiv preprint arXiv:2201.05675}, 2022.

\bibitem[Xu and Zheng(2024)]{xu2024efficient}
Angchi Xu and Wei-Shi Zheng.
\newblock Efficient and effective weakly-supervised action segmentation via action-transition-aware boundary alignment.
\newblock \emph{arXiv preprint arXiv:2403.19225}, 2024.

\bibitem[Ibrahim et~al.(2016)Ibrahim, Muralidharan, Deng, Vahdat, and Mori]{ibrahim2016hierarchical}
Mostafa~S Ibrahim, Srikanth Muralidharan, Zhiwei Deng, Arash Vahdat, and Greg Mori.
\newblock A hierarchical deep temporal model for group activity recognition.
\newblock In \emph{Proceedings of the IEEE conference on computer vision and pattern recognition}, pages 1971--1980, 2016.

\bibitem[Gavrilyuk et~al.(2020)Gavrilyuk, Sanford, Javan, and Snoek]{gavrilyuk2020actor}
Kirill Gavrilyuk, Ryan Sanford, Mehrsan Javan, and Cees~GM Snoek.
\newblock Actor-transformers for group activity recognition.
\newblock In \emph{Proceedings of the IEEE/CVF Conference on computer vision and pattern recognition}, pages 839--848, 2020.

\bibitem[Kim et~al.(2022)Kim, Lee, Cho, and Kwak]{kim2022detector}
Dongkeun Kim, Jinsung Lee, Minsu Cho, and Suha Kwak.
\newblock Detector-free weakly supervised group activity recognition.
\newblock In \emph{Proceedings of the IEEE/CVF Conference on Computer Vision and Pattern Recognition}, pages 20083--20093, 2022.

\bibitem[Huang et~al.(2022{\natexlab{a}})Huang, Feng, Lu, Magliacane, and Zhang]{huang2022adarl}
Biwei Huang, Fan Feng, Chaochao Lu, Sara Magliacane, and Kun Zhang.
\newblock Ada{RL}: What, where, and how to adapt in transfer reinforcement learning.
\newblock In \emph{International Conference on Learning Representations}, 2022{\natexlab{a}}.
\newblock URL \url{https://openreview.net/forum?id=8H5bpVwvt5}.

\bibitem[Zheng et~al.(2022)Zheng, Ng, and Zhang]{zheng2022on}
Yujia Zheng, Ignavier Ng, and Kun Zhang.
\newblock On the identifiability of nonlinear {ICA}: Sparsity and beyond.
\newblock In Alice~H. Oh, Alekh Agarwal, Danielle Belgrave, and Kyunghyun Cho, editors, \emph{Advances in Neural Information Processing Systems}, 2022.
\newblock URL \url{https://openreview.net/forum?id=Wo1HF2wWNZb}.

\bibitem[Kingma and Welling(2013)]{kingma2013auto}
Diederik~P Kingma and Max Welling.
\newblock Auto-encoding variational bayes.
\newblock \emph{arXiv preprint arXiv:1312.6114}, 2013.

\bibitem[Yingzhen and Mandt(2018)]{yingzhen2018disentangled}
Li~Yingzhen and Stephan Mandt.
\newblock Disentangled sequential autoencoder.
\newblock In \emph{International Conference on Machine Learning}, pages 5670--5679. PMLR, 2018.

\bibitem[Higgins et~al.(2017)Higgins, Matthey, Pal, Burgess, Glorot, Botvinick, Mohamed, and Lerchner]{higgins2017betavae}
Irina Higgins, Loic Matthey, Arka Pal, Christopher Burgess, Xavier Glorot, Matthew Botvinick, Shakir Mohamed, and Alexander Lerchner.
\newblock beta-{VAE}: Learning basic visual concepts with a constrained variational framework.
\newblock In \emph{International Conference on Learning Representations}, 2017.
\newblock URL \url{https://openreview.net/forum?id=Sy2fzU9gl}.

\bibitem[Souri et~al.(2021{\natexlab{a}})Souri, Fayyaz, Minciullo, Francesca, and Gall]{mucon2021}
Yaser Souri, Mohsen Fayyaz, Luca Minciullo, Gianpiero Francesca, and Juergen Gall.
\newblock {Fast Weakly Supervised Action Segmentation Using Mutual Consistency}.
\newblock \emph{{PAMI}}, 2021{\natexlab{a}}.

\bibitem[Bojanowski et~al.(2014)Bojanowski, Lajugie, Bach, Laptev, Ponce, Schmid, and Sivic]{bojanowski2014weakly}
Piotr Bojanowski, R{\'e}mi Lajugie, Francis Bach, Ivan Laptev, Jean Ponce, Cordelia Schmid, and Josef Sivic.
\newblock Weakly supervised action labeling in videos under ordering constraints.
\newblock In \emph{Computer Vision--ECCV 2014: 13th European Conference, Zurich, Switzerland, September 6-12, 2014, Proceedings, Part V 13}, pages 628--643. Springer, 2014.

\bibitem[Zhukov et~al.(2019)Zhukov, Alayrac, Cinbis, Fouhey, Laptev, and Sivic]{zhukov2019cross}
Dimitri Zhukov, Jean-Baptiste Alayrac, Ramazan~Gokberk Cinbis, David Fouhey, Ivan Laptev, and Josef Sivic.
\newblock Cross-task weakly supervised learning from instructional videos.
\newblock In \emph{Proceedings of the IEEE/CVF Conference on Computer Vision and Pattern Recognition}, pages 3537--3545, 2019.

\bibitem[Lu and Elhamifar(2021{\natexlab{a}})]{Lu_2021_ICCV}
Zijia Lu and Ehsan Elhamifar.
\newblock Weakly-supervised action segmentation and alignment via transcript-aware union-of-subspaces learning.
\newblock In \emph{Proceedings of the IEEE/CVF International Conference on Computer Vision (ICCV)}, pages 8085--8095, October 2021{\natexlab{a}}.

\bibitem[Richard et~al.(2017)Richard, Kuehne, and Gall]{richard2017weakly}
Alexander Richard, Hilde Kuehne, and Juergen Gall.
\newblock Weakly supervised action learning with rnn based fine-to-coarse modeling.
\newblock In \emph{Proceedings of the IEEE conference on Computer Vision and Pattern Recognition}, pages 754--763, 2017.

\bibitem[Li et~al.(2019)Li, Lei, and Todorovic]{li2019weakly}
Jun Li, Peng Lei, and Sinisa Todorovic.
\newblock Weakly supervised energy-based learning for action segmentation.
\newblock In \emph{Proceedings of the IEEE/CVF international conference on computer vision}, pages 6243--6251, 2019.

\bibitem[Richard et~al.(2018)Richard, Kuehne, Iqbal, and Gall]{richard2018neuralnetwork}
Alexander Richard, Hilde Kuehne, Ahsan Iqbal, and Juergen Gall.
\newblock Neuralnetwork-viterbi: A framework for weakly supervised video learning.
\newblock In \emph{Proceedings of the IEEE conference on Computer Vision and Pattern Recognition}, pages 7386--7395, 2018.

\bibitem[Lu and Elhamifar(2022)]{lu2022set}
Zijia Lu and Ehsan Elhamifar.
\newblock Set-supervised action learning in procedural task videos via pairwise order consistency.
\newblock In \emph{Proceedings of the IEEE/CVF Conference on Computer Vision and Pattern Recognition}, pages 19903--19913, 2022.

\bibitem[Hwang et~al.(2024)Hwang, Kwak, Choi, Zhang, and Lee]{hwang2024quantized}
Inwoo Hwang, Yunhyeok Kwak, Suhyung Choi, Byoung-Tak Zhang, and Sanghack Lee.
\newblock Quantized local independence discovery for fine-grained causal dynamics learning in reinforcement learning, 2024.
\newblock URL \url{https://openreview.net/forum?id=9UGAUQjibp}.

\bibitem[Silva et~al.(2006)Silva, Scheines, Glymour, Spirtes, and Chickering]{silva2006learning}
Ricardo Silva, Richard Scheines, Clark Glymour, Peter Spirtes, and David~Maxwell Chickering.
\newblock Learning the structure of linear latent variable models.
\newblock \emph{Journal of Machine Learning Research}, 7\penalty0 (2), 2006.

\bibitem[Kummerfeld and Ramsey(2016)]{kummerfeld2016causal}
Erich Kummerfeld and Joseph Ramsey.
\newblock Causal clustering for 1-factor measurement models.
\newblock In \emph{Proceedings of the 22nd ACM SIGKDD international conference on knowledge discovery and data mining}, pages 1655--1664, 2016.

\bibitem[Huang et~al.(2022{\natexlab{b}})Huang, Low, Xie, Glymour, and Zhang]{huang2022latent}
Biwei Huang, Charles Jia~Han Low, Feng Xie, Clark Glymour, and Kun Zhang.
\newblock Latent hierarchical causal structure discovery with rank constraints.
\newblock \emph{Advances in Neural Information Processing Systems}, 35:\penalty0 5549--5561, 2022{\natexlab{b}}.

\bibitem[Spearman(1928)]{spearman1928pearson}
Charles Spearman.
\newblock Pearson's contribution to the theory of two factors.
\newblock \emph{British Journal of Psychology}, 19\penalty0 (1):\penalty0 95, 1928.

\bibitem[Shimizu et~al.(2009)Shimizu, Hoyer, and Hyv{\"a}rinen]{shimizu2009estimation}
Shohei Shimizu, Patrik~O Hoyer, and Aapo Hyv{\"a}rinen.
\newblock Estimation of linear non-gaussian acyclic models for latent factors.
\newblock \emph{Neurocomputing}, 72\penalty0 (7-9):\penalty0 2024--2027, 2009.

\bibitem[Cai and Xie(2019)]{cai2019triad}
Ruichu Cai and Feng Xie.
\newblock Triad constraints for learning causal structure of latent variables.
\newblock \emph{Advances in neural information processing systems}, 2019.

\bibitem[Xie et~al.(2020)Xie, Cai, Huang, Glymour, Hao, and Zhang]{xie2020generalized}
Feng Xie, Ruichu Cai, Biwei Huang, Clark Glymour, Zhifeng Hao, and Kun Zhang.
\newblock Generalized independent noise condition for estimating latent variable causal graphs.
\newblock \emph{arXiv preprint arXiv:2010.04917}, 2020.

\bibitem[Xie et~al.(2022)Xie, Huang, Chen, He, Geng, and Zhang]{xie2022identification}
Feng Xie, Biwei Huang, Zhengming Chen, Yangbo He, Zhi Geng, and Kun Zhang.
\newblock Identification of linear non-gaussian latent hierarchical structure.
\newblock In \emph{International Conference on Machine Learning}, pages 24370--24387. PMLR, 2022.

\bibitem[Pearl(1988)]{pearl1988probabilistic}
Judea Pearl.
\newblock \emph{Probabilistic reasoning in intelligent systems: networks of plausible inference}.
\newblock Morgan kaufmann, 1988.

\bibitem[Zhang(2004)]{zhang2004hierarchical}
Nevin~L Zhang.
\newblock Hierarchical latent class models for cluster analysis.
\newblock \emph{The Journal of Machine Learning Research}, 5:\penalty0 697--723, 2004.

\bibitem[Choi et~al.(2011)Choi, Tan, Anandkumar, and Willsky]{choi2011learning}
Myung~Jin Choi, Vincent~YF Tan, Animashree Anandkumar, and Alan~S Willsky.
\newblock Learning latent tree graphical models.
\newblock \emph{Journal of Machine Learning Research}, 12:\penalty0 1771--1812, 2011.

\bibitem[Drton et~al.(2015)Drton, Lin, Weihs, and Zwiernik]{drton2017marginal}
Mathias Drton, Shaowei Lin, Luca Weihs, and Piotr Zwiernik.
\newblock Marginal likelihood and model selection for gaussian latent tree and forest models, 2015.

\bibitem[Lippe et~al.(2022)Lippe, Magliacane, L{\"o}we, Asano, Cohen, and Gavves]{lippe2022citris}
Phillip Lippe, Sara Magliacane, Sindy L{\"o}we, Yuki~M Asano, Taco Cohen, and Efstratios Gavves.
\newblock {CITRIS}: Causal identifiability from temporal intervened sequences.
\newblock In \emph{ICLR2022 Workshop on the Elements of Reasoning: Objects, Structure and Causality}, 2022.
\newblock URL \url{https://openreview.net/forum?id=H87xrH_Lcg9}.

\bibitem[Lippe et~al.(2023)Lippe, Magliacane, L{\"o}we, Asano, Cohen, and Gavves]{lippe2023causal}
Phillip Lippe, Sara Magliacane, Sindy L{\"o}we, Yuki~M Asano, Taco Cohen, and Efstratios Gavves.
\newblock Causal representation learning for instantaneous and temporal effects in interactive systems.
\newblock In \emph{The Eleventh International Conference on Learning Representations}, 2023.
\newblock URL \url{https://openreview.net/forum?id=itZ6ggvMnzS}.

\bibitem[Chen et~al.(2024)Chen, Shen, Chen, Song, Sun, Yao, Liu, and Zhang]{chen2024caring}
Guangyi Chen, Yifan Shen, Zhenhao Chen, Xiangchen Song, Yuewen Sun, Weiran Yao, Xiao Liu, and Kun Zhang.
\newblock Caring: Learning temporal causal representation under non-invertible generation process, 2024.

\bibitem[Huang et~al.(2016)Huang, Fei-Fei, and Niebles]{huang2016connectionist}
De-An Huang, Li~Fei-Fei, and Juan~Carlos Niebles.
\newblock Connectionist temporal modeling for weakly supervised action labeling.
\newblock In \emph{Proceedings of the European Conference on Computer Vision}, pages 137--153. Springer, 2016.

\bibitem[Ding and Xu(2018)]{ding2018weakly}
Li~Ding and Chenliang Xu.
\newblock Weakly-supervised action segmentation with iterative soft boundary assignment.
\newblock In \emph{Proceedings of the IEEE conference on computer vision and pattern recognition}, pages 6508--6516, 2018.

\bibitem[Chang et~al.(2019)Chang, Huang, Sui, Fei-Fei, and Niebles]{chang2019d3tw}
Chien-Yi Chang, De-An Huang, Yanan Sui, Li~Fei-Fei, and Juan~Carlos Niebles.
\newblock D3tw: Discriminative differentiable dynamic time warping for weakly supervised action alignment and segmentation.
\newblock In \emph{Proceedings of the IEEE/CVF Conference on Computer Vision and Pattern Recognition}, pages 3546--3555, 2019.

\bibitem[Souri et~al.(2021{\natexlab{b}})Souri, Fayyaz, Minciullo, Francesca, and Gall]{souri2021fast}
Yaser Souri, Mohsen Fayyaz, Luca Minciullo, Gianpiero Francesca, and Juergen Gall.
\newblock Fast weakly supervised action segmentation using mutual consistency.
\newblock \emph{IEEE Transactions on Pattern Analysis and Machine Intelligence}, 44\penalty0 (10):\penalty0 6196--6208, 2021{\natexlab{b}}.

\bibitem[Chang et~al.(2021)Chang, Tung, and Mori]{chang2021learning}
Xiaobin Chang, Frederick Tung, and Greg Mori.
\newblock Learning discriminative prototypes with dynamic time warping.
\newblock In \emph{Proceedings of the IEEE/CVF Conference on Computer Vision and Pattern Recognition}, pages 8395--8404, 2021.

\bibitem[Lu and Elhamifar(2021{\natexlab{b}})]{lu2021weakly}
Zijia Lu and Ehsan Elhamifar.
\newblock Weakly-supervised action segmentation and alignment via transcript-aware union-of-subspaces learning.
\newblock In \emph{Proceedings of the IEEE/CVF International Conference on Computer Vision}, pages 8085--8095, 2021{\natexlab{b}}.

\bibitem[Kuehne et~al.(2017)Kuehne, Richard, and Gall]{kuehne2017weakly}
Hilde Kuehne, Alexander Richard, and Juergen Gall.
\newblock Weakly supervised learning of actions from transcripts.
\newblock \emph{Computer Vision and Image Understanding}, 163:\penalty0 78--89, 2017.

\bibitem[Souri et~al.(2022)Souri, Farha, Despinoy, Francesca, and Gall]{souri2022fifa}
Yaser Souri, Yazan~Abu Farha, Fabien Despinoy, Gianpiero Francesca, and Juergen Gall.
\newblock Fifa: Fast inference approximation for action segmentation.
\newblock In \emph{Proceedings of the DAGM German Conference on Pattern Recognition}, pages 282--296. Springer, 2022.

\bibitem[Kuehne et~al.(2018)Kuehne, Richard, and Gall]{kuehne2018hybrid}
Hilde Kuehne, Alexander Richard, and Juergen Gall.
\newblock A hybrid rnn-hmm approach for weakly supervised temporal action segmentation.
\newblock \emph{IEEE transactions on pattern analysis and machine intelligence}, 42\penalty0 (4):\penalty0 765--779, 2018.

\bibitem[Zhang et~al.(2023)Zhang, Wang, Duan, Tang, Zhang, and Tan]{zhang2023hoi}
Runzhong Zhang, Suchen Wang, Yueqi Duan, Yansong Tang, Yue Zhang, and Yap-Peng Tan.
\newblock Hoi-aware adaptive network for weakly-supervised action segmentation.
\newblock In \emph{Proceedings of the Thirty-Second International Joint Conference on Artificial Intelligence}, pages 1722--1730, 2023.

\end{thebibliography}
\bibliographystyle{unsrtnat}

\newpage
\appendix
  \textit{\large Supplement to}\\ \ \\
      {\Large \bf ``Causal Temporal Representation Learning with \\ Nonstationary Sparse Transition''}\
\newcommand{\beginsupplement}{%
	\setcounter{table}{0}
	\renewcommand{\thetable}{S\arabic{table}}%
	\setcounter{figure}{0}
	\renewcommand{\thefigure}{S\arabic{figure}}%
	\setcounter{section}{0}
	\renewcommand{\thesection}{S\arabic{section}}%
	\setcounter{theorem}{0}
	\renewcommand{\thetheorem}{S\arabic{theorem}}%
	\setcounter{proposition}{0}
	\renewcommand{\theproposition}{S\arabic{proposition}}%
	\setcounter{corollary}{0}
	\renewcommand{\thecorollary}{S\arabic{corollary}}%
	\setcounter{lemma}{0}
	\renewcommand{\thelemma}{S\arabic{lemma}}%
}

\beginsupplement
{\large Appendix organization:}

\DoToC 
\clearpage

\section{Identifiability Theory}
\label{ap: Appendix Identifiability}

\subsection{Proof for Theorem~\ref{thm: identifiability of C}}
\label{ap: Appendix Proof for Theorem identifiability of C}
We divide the complete proof into two principal steps: 
\begin{enumerate}
    \item Firstly, assuming access to the optimal mixing function estimation \(\hat{\mathbf{g}}^{*}\), we demonstrate that under the conditions in our theorem, the estimated clustering result will align with the ground truth up to label swapping. This alignment is due to the transition complexity with optimal \(\hat{u}^{*}_t\) and \(\hat{\mathbf{g}}^{*}\) being strictly lower than that with non-optimal \(\hat{u}_t\) but still optimal \(\hat{\mathbf{g}}^{*}\).
    \item Secondly, we generalize the results of the first step to cases where the mixing function estimation \(\hat{\mathbf{g}}\) is suboptimal. We establish that for any given clustering assignment, whether optimal or not, a suboptimal mixing function estimation \(\hat{\mathbf{g}}\) can not result in a lower transition complexity. Thus, the transition complexity in scenarios with non-optimal \(\hat{\mathbf{g}}\) will always be at least as high as in the optimal case. 
\end{enumerate}
From those two steps, we conclude that the global minimum transition complexity can only be attained when the estimation of domain variables \(\hat{u}_t\) is optimal, hence ensuring that the estimated clustering must match the ground truth up to label swapping. It is important to note that this condition alone does not guarantee the identifiability of the mixing function \(\mathbf{g}\). Because a setting with optimal \(\hat{u}^{*}_t\) and a non-optimal \(\hat{\mathbf{g}}\) may exhibit equivalent transition complexity to the optimal scenario, but it does not compromise our proof for the identifiability of domain variables \(u_t\). Further exploration of the mixing function's identifiability \(\mathbf{g}\) is discussed in Theorem~\ref{thm: identifiability of z} in the subsequent section.

\subsubsection{Identifiability of \texorpdfstring{$\mathcal{C}$}{C} under optimal \texorpdfstring{$\hat{\mathbf{g}}^{*}$}{g}}
\begin{wrapfigure}{r}{0.37\textwidth}
\vspace{-4em}
    \centering
    \includegraphics[width=0.35\textwidth]{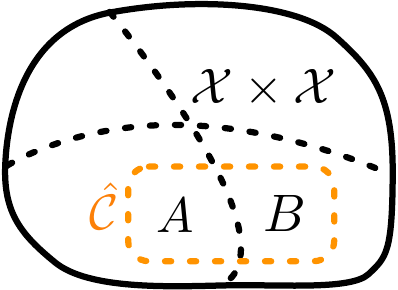} 
    \caption{Illustration of \(\hat{\mathcal{C}}\) incorrectly assigning two different domain subsets of inputs \(A\) and \(B\) into the same \(\hat{u}\). The black lines represent the ground truth partition of \(\mathcal{C}\) and the orange line represent the incorrect domain partition for set \(A\) and \(B\).}
    \label{fig: wrong C}
    \vspace{-3em}
\end{wrapfigure}
We fist introduce a lemma for this case when we can access an optimal mixing function estimation \(\hat{\mathbf{g}}^{*}\).

\begin{lemma}[Identifiability of $\mathcal{C}$ under optimal $\hat{\mathbf{g}}^{*}$]\label{lemma: id of C under optimal g}
    In addition to the assumptions in Theorem~\ref{thm: identifiability of C}, assume that we can also access an optimal estimation of \(\mathbf{g}\), denoted by \(\hat{\mathbf{g}}^{*}\), in which the estimated \(\hat{\mathbf{z}}_t\) is an invertible, component-wise transformation of a permuted version of \(\mathbf{z}_t\). Then the estimated clustering \(\hat{\mathcal{C}}\) must match the ground truth up to label swapping. 
\end{lemma}
\begin{proof}
In the first case we deal with optimal estimation \(\hat{\mathbf{g}}^{*}\) in which the estimated \(\hat{\mathbf{z}}_t\) is an invertible, component-wise transformation of a permuted version of \(\mathbf{z}_t\), but inaccurate estimated version of \(\hat{\mathcal{C}}\), Consider the following example (Figure \ref{fig: wrong C}):
\label{ap: id of C}

With slight abuse of notation, we use \(\mathcal{C}(A)\) to represent the domain assigned by \(\mathcal{C}\) to all elements in \(A\), and all elements in \(A\) have the same assignment. The same argument applies to \(B\).

Then, for an estimated \(\hat{\mathcal{C}}\), if it incorrectly assigns two subsets of input \(A\) and \(B\) to the same \(\hat{u}\) (Figure \ref{fig: wrong C} orange circle), i.e., 
\begin{equation}\label{eq: incorrect C}
    \mathcal{C}(A) = i \neq j = \mathcal{C}(B) \quad \text{but} \quad \hat{\mathcal{C}}(A) = \hat{C}(B) = k.
\end{equation}
Note that if the ground truth \(\mathcal{C}\) gives a consistent assignment for \(A\) and \(B\) but estimated \(\hat{\mathcal{C}}\) gives diverse assignments, i.e. \begin{equation}
    \hat{\mathcal{C}}(A) = i \neq j = \hat{\mathcal{C}}(B) \quad \text{but} \quad \mathcal{C}(A) = \mathcal{C}(B) = k,
\end{equation} it is nothing but further splitting the ground truth assignment in a more fine-grained manner.
This scenario does not break the boundaries of the ground truth assignments. Consider two cases in the estimation process:
\begin{enumerate}
    \item If the number of allowed regimes or domains exceeds that of the ground truth, such more fine-grained assignment is allowed. The ground truth can then be easily recovered by merging domains that share identical Jacobian supports.
    \item If the number of regimes or domains matches the ground truth, it can be shown that the inconsistent scenario outlined in Eq.~\eqref{eq: incorrect C} must occur.
\end{enumerate}
Given that these considerations do not directly affect our approach, they are omitted from further discussion for brevity.
\begin{figure}[htbp]
\centering
\resizebox{0.8\textwidth}{!}{\(
    \begin{array}{ccc}
    
    \mathcal{M}_i = 
    \begin{bmatrix}
    \textcolor{red}{0} & 0 & 1 & 0 \\
    1 & 1 & 0 & 0 \\
    0 & 0 & 0 & 1 \\
    0 & 0 & 1 & 0 \\
    \end{bmatrix}
    &
    \mathcal{M}_j = 
    \begin{bmatrix}
    \textcolor{red}{1} & 0 & 1 & 0 \\
    1 & 1 & 0 & 0 \\
    0 & 0 & 0 & 1 \\
    0 & 0 & 1 & 0 \\
    \end{bmatrix}
    &
    \widehat{\mathcal{M}}_k = 
    \begin{bmatrix}
    \textcolor{red}{1} & 0 & 1 & 0 \\
    1 & 1 & 0 & 0 \\
    0 & 0 & 0 & 1 \\
    0 & 0 & 1 & 0 \\
    \end{bmatrix} 
    
    \end{array}
\)}

\caption{Comparison of matrices $\mathcal{M}_i$, $\mathcal{M}_j$, and $\widehat{\mathcal{M}}_k$. The elements in red highlight the differences between them.}
\label{fig: comparison of Mi, Mj, M_hat_k}
\end{figure}

Then considering the case in Eq.~\ref{eq: incorrect C}, the estimated transition must cover the functions from both \(A\) and \(B\), then the learned transition \(\hat{\mathbf{m}}_k \) must have Jacobian \(\mathbf{J}_{\hat{\mathbf{m}}_k}\) with support matrix \(\widehat{\mathcal{M}}_k = \mathcal{M}_i + \mathcal{M}_j \) which is the binary addition of \(\mathcal{M}_i\) and \(\mathcal{M}_j\), such that for all indices in \(\mathcal{M}_i, \mathcal{M}_j\) if any of these two is \(1\), then the corresponding position in \(\widehat{\mathcal{M}}_k\) must be \(1\). That is because if that is not the case, for example, the \((a,b)\)-th location for \(\mathcal{M}_i\), \(\mathcal{M}_j\), and \(\widehat{\mathcal{M}}_k\) are 1, 0, and 0. Then we can easily find an input region for the \((a,b)\)-th location such that a small perturbation can lead to changes in \(\mathbf{m}_i\) but not in \(\mathbf{m}_j\) nor \(\hat{\mathbf{m}}_k\), which makes \(\hat{\mathbf{m}}_k\) unable to fit all of the transitions in \(A\cup B\) which cause contradiction. See the three matrices in Figure \ref{fig: comparison of Mi, Mj, M_hat_k} for an illustrative example.

By assumption~\ref{as: Mechanism Variability}, since all those support matrix differ at least one spot, which means the estimated version is not smaller than the ground truth. 
\begin{equation}\label{eq: >= M}
| \widehat{\mathcal{M}}_k | \geq | \mathcal{M}_j |\quad \text{and} \quad| \widehat{\mathcal{M}}_k | \geq |\mathcal{M}_i |,
\end{equation}
and the equality cannot hold true at the same time. 

Then from Assumption~\eqref{as: Mechanism Separability}, the expected estimated transition complexity can be expressed as:
\begin{equation}
\mathbb{E}_{\mathcal{D}} |\widehat{\mathcal{M}}_{\hat{u}}| = \int_{\mathcal{X} \times \mathcal{X}} p_{\mathcal{D}}(\mathbf{x}_{t-1}, \mathbf{x}_t) \cdot | \widehat{\mathcal{M}}_{\hat{\mathcal{C}}(\mathbf{x}_{t-1}, \mathbf{x}_t)} | \, d\mathbf{x}_{t-1} \, d\mathbf{x}_{t}.
\end{equation}
Similarly for ground truth one:
\begin{equation}
\mathbb{E}_{\mathcal{D}} |\mathcal{M}_{u}| = \int_{\mathcal{X} \times \mathcal{X}} p_{\mathcal{D}}(\mathbf{x}_{t-1}, \mathbf{x}_t) \cdot | \mathcal{M}_{\mathcal{C}(\mathbf{x}_{t-1}, \mathbf{x}_t)} | \, d\mathbf{x}_{t-1} \, d\mathbf{x}_{t}.
\end{equation}
Let us focus on the integral of the region \(A\cup B\), the subset of \(\mathcal{X} \times \mathcal{X}\) mentioned above.
If for some area \(p_{\mathcal{D}}(\mathbf{x}_{t-1}, \mathbf{x}_t) = 0\), then the clustering under this area is ill defined since there is no support from data. Hence we only need to deal with supported area.
For area that \(p_{\mathcal{D}}(\mathbf{x}_{t-1}, \mathbf{x}_t) >0\) and from Eq.~\eqref{eq: >= M} the equality cannot hold true at the same time, then the estimated version of the integral is strictly larger than the ground-truth version for any inconsistent clustering as indicated in Eq.~\eqref{eq: incorrect C}.

For the rest of regions in $\mathcal{X} \times \mathcal{X}$, any incorrect cluster assignment will further increase the \(\mathcal{M}\) with same reason as discussed above, then the estimated complexity is strictly larger than the ground truth complexity:
\begin{equation}
\mathbb{E}_{\mathcal{D}} | \widehat{\mathcal{M}}_{\hat{u}} | > \mathbb{E}_{\mathcal{D}} | \mathcal{M}_u |. \label{eq: identifiability of C larger |M|}
\end{equation}

But assumption~\eqref{as: Mechanism Sparsity} requires that the estimated complexity be less than or equal to the ground truth. Contradiction! Hence, the estimated \(\hat{\mathcal{C}}\) must match the ground truth up to label swapping.
\end{proof}

\subsubsection{Identifiability of \texorpdfstring{$\mathcal{C}$}{C} under arbitrary \texorpdfstring{\(\mathbf{g}\)}{g}}\label{ap: id of C under wrong g}
Now we can leverage the conclusion in Lemma~\ref{lemma: id of C under optimal g} to show the identifiability of domain variables under arbitrary mixing function estimation.
\begin{theorem}[Identifiability of Domain Variables]
Suppose that the data \(\mathcal{D}\) are generated from the nonstationary data generation process as described in Eqs. \eqref{Eq:mixing function} and \eqref{Eq:transition function of z}. Suppose the transitions are weakly diverse lossy~(Def.~\ref{def:Diversely Lossy Transformation}) and the following assumptions hold:
\begin{enumerate}[label=\roman*.,ref=\roman*]
    \item \underline{(Mechanism Separability)} There exists a ground truth mapping \(\mathcal{C}: \mathcal{X} \times \mathcal{X}  \to \mathcal{U}\) determined the real domain indices, i.e., \(u_t = \mathcal{C}(\mathbf{x}_{t-1}, \mathbf{x}_{t})\).
    \item \underline{(Mechanism Sparsity)} The estimated transition complexity on dataset \(\mathcal{D}\) is less than or equal to ground truth transition complexity, i.e., \(\mathbb{E}_{\mathcal{D}} | \widehat{\mathcal{M}}_{\hat{u}} | \leq \mathbb{E}_{\mathcal{D}}  | \mathcal{M}_{u} |\).
    \item \underline{(Mechanism Variability)} Mechanisms are sufficiently different. For all $u\neq u'$, $\mathcal{M}_{u} \neq \mathcal{M}_{u'}$ i.e. there exists index $(i,j)$ such that \(\left[\mathcal{M}_{u}\right]_{i,j} \neq \left[\mathcal{M}_{u'}\right]_{i,j}\).
\end{enumerate}
Then the domain variables $u_t$ is identifiable up to label swapping (Def.~\ref{def: Identifiable Domain Variables}).
\end{theorem}

\begin{proof}
To demonstrate the complete identifiability of $\mathcal{C}$, independent of the estimation quality of $\mathbf{g}$, we must show that for any arbitrary estimation $\hat{\mathcal{C}} \neq \sigma(\mathcal{C})$, the induced $\hat{\mathcal{M}}_{\hat{u}}$ for inaccurate estimation $\hat{\mathbf{g}}$ has a transition complexity at least as high as in the optimal $\hat{\mathbf{g}}^{*}$ case. If this holds, from Lemma~\ref{lemma: id of C under optimal g}, we can conclude that the transition complexity of optimal $\hat{\mathcal{C}}^{*} = \sigma(\mathcal{C})$ and optimal $\hat{\mathbf{g}}^{*}$ is strictly smaller than any non-optimal $\hat{\mathcal{C}}$ and any $\hat{\mathbf{g}}$.

Suppose the estimated decoder and corresponding latent variables are $\hat{\mathbf{g}}$ and $\hat{\mathbf{z}}_t$, respectively, then the following relation holds:
\begin{equation}
    \hat{\mathbf{g}}^{*}(\mathbf{z}_t) = \hat{\mathbf{g}}(\hat{\mathbf{z}}_t).
    \label{eq: g = g hat}    
\end{equation}
Since $\hat{\mathbf{g}}$ is invertible, by composing $\hat{\mathbf{g}}^{-1}$ on both sides, we obtain:
\begin{equation}
    \hat{\mathbf{g}}^{-1} \circ \hat{\mathbf{g}}^{*}(\mathbf{z}_t) = \hat{\mathbf{g}}^{-1} \circ \hat{\mathbf{g}}(\hat{\mathbf{z}}_t).
\end{equation}
Let 
\begin{equation}
    \mathbf{h} \coloneq \hat{\mathbf{g}}^{-1} \circ \hat{\mathbf{g}}^{*},
    \label{eq: h}
\end{equation}
we then have:
\begin{equation}\label{eq:h}
    \mathbf{h}(\hat{\mathbf{z}}^{*}_t) = \hat{\mathbf{z}}_t.
\end{equation}
We aim to demonstrate that under this transformation, if $\mathbf{h}$ is not a permutation and component-wise transformation, the introduced transition complexity among estimated $\hat{\mathbf{z}}$ will not be smaller than the optimal $\hat{\mathbf{g}}^{*}$.

\begin{proposition}\label{prop: s1}
Suppose \(|\widehat{\mathcal{M}}| < |\widehat{\mathcal{M}}^{*}|\), then for any permutation \(\sigma\) mapping the indices of the dimensions from \(\widehat{\mathcal{M}}\) to \(\widehat{\mathcal{M}}^{*}\), there must exist an index pair \((i,j)\) such that \(\widehat{\mathcal{M}}_{i,j} = 0\) and \(\widehat{\mathcal{M}}^{*}_{\sigma(i),\sigma(j)} = 1\).
\end{proposition}

\begin{proof}
An intuitive explanation for this proposition involves the construction of a directed graph \(G_{\widehat{\mathcal{M}}^{*}} = (V_{\widehat{\mathcal{M}}^{*}}, E_{\widehat{\mathcal{M}}^{*}})\), where \(V_{\widehat{\mathcal{M}}^{*}} = \{1, 2, \dots, n\}\) and \(E_{\widehat{\mathcal{M}}^{*}} = \{(i,j) \mid \widehat{\mathcal{M}}^{*}_{i,j} = 1\}\). A similar construction can be made for \(G_{\widehat{\mathcal{M}}}\). It is straightforward that \(|\widehat{\mathcal{M}}^{*}| = |E_{\widehat{\mathcal{M}}^{*}}|\), which represents the number of edges. Consequently, \(|\widehat{\mathcal{M}}| < |\widehat{\mathcal{M}}^{*}|\) implies that \(G_{\widehat{\mathcal{M}}^{*}}\) has more edges than \(G_{\widehat{\mathcal{M}}}\). Since there is no pre-defined ordering information for the nodes in these two graphs, if we wish to compare their edges, we need to first establish a mapping. However, if \(|E_{\widehat{\mathcal{M}}}| < |E_{\widehat{\mathcal{M}}^{*}}|\), no matter how the mapping \(\sigma\) is constructed, there must be an index pair \((i,j)\) such that \((i,j) \notin E_{\widehat{\mathcal{M}}}\) but \((\sigma(i),\sigma(j)) \in E_{\widehat{\mathcal{M}}^{*}}\). Otherwise, if such an index pair does not exist, the total number of edges in \(G_{\widehat{\mathcal{M}}}\) would necessarily be greater than or equal to that in \(G_{\widehat{\mathcal{M}}^{*}}\), contradicting the premise that \(|\widehat{\mathcal{M}}| < |\widehat{\mathcal{M}}^{*}|\).
\end{proof}

\begin{lemma}[Non-decreasing Complexity]
    Suppose transitions are weakly diverse lossy as defined in Def.~\ref{def:Diversely Lossy Transformation} and an invertible transformation $\mathbf{h}$ maps the optimal estimation $\hat{\mathbf{z}}^{*}_t$ to the estimated $\hat{\mathbf{z}}_t$, and it is neither a permutation nor a component-wise transformation. Then, the transition complexity on the estimated $\hat{\mathbf{z}}_t$ is not lower than that on the optimal $\hat{\mathbf{z}}^{*}_t$, i.e., 
    \[
    |\widehat{\mathcal{M}}| \geq |\widehat{\mathcal{M}}^{*}|.
    \]
\end{lemma}
\begin{proof}

\begin{figure}[ht]
    \centering
    \begin{subfigure}{0.45\textwidth}
        \centering
        \includegraphics[width=0.8\linewidth]{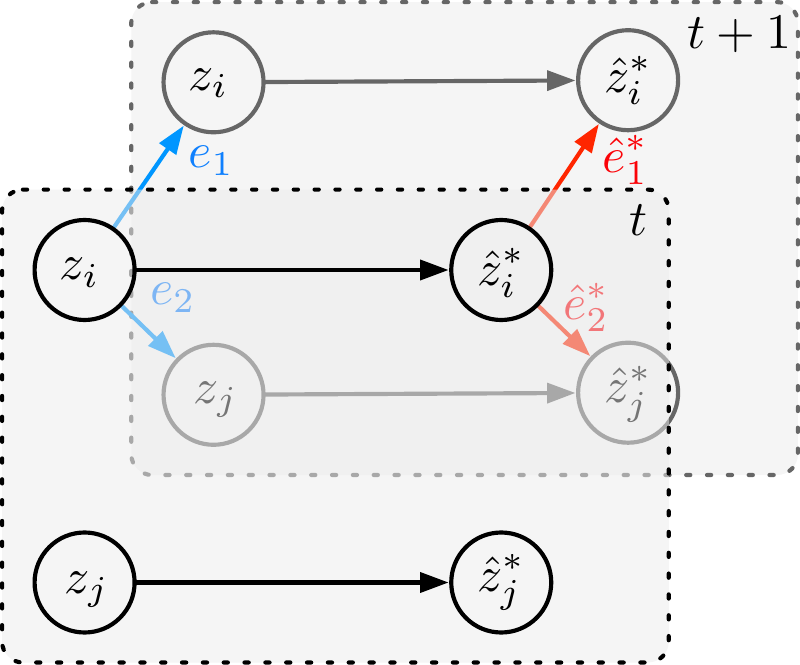}
        \caption{Transition graph for ground truth \(\mathbf{z}_t\) and optimal estimation \(\hat{\mathbf{z}}^{*}_t\)}
        \label{fig:lemma-opt}
    \end{subfigure}%
    \hspace{1cm}
    \begin{subfigure}{0.45\textwidth}
        \centering
        \includegraphics[width=0.8\linewidth]{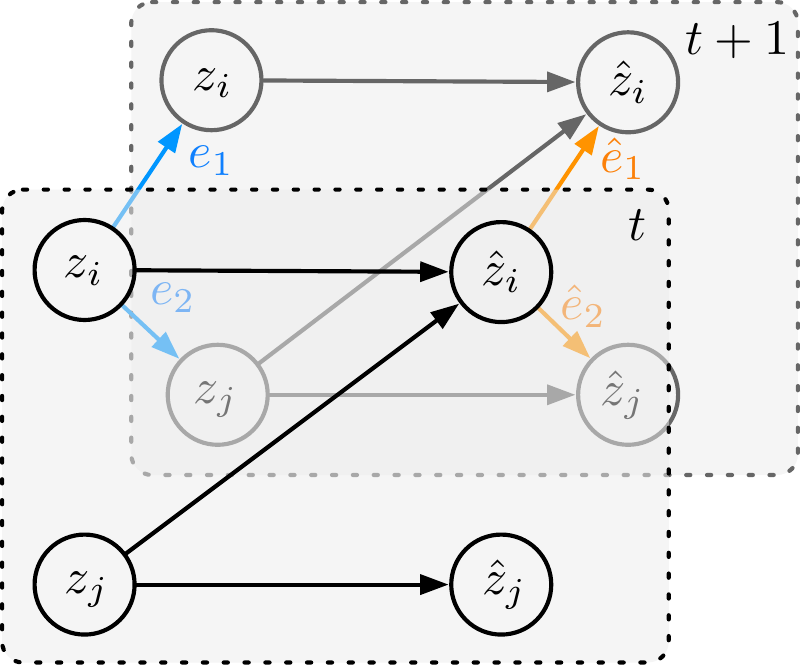}
        \caption{Transition graph for ground truth \(\mathbf{z}_t\) and  arbitrary estimation \(\hat{\mathbf{z}}_t\)}
        \label{fig:lemma-est}
    \end{subfigure}
    \caption{A partial observation of the transition graph among ground truth \(\mathbf{z}_t\), optimal estimation \(\hat{\mathbf{z}}^{*}_t\) and arbitrary estimation \(\hat{\mathbf{z}}_t\). For brevity, the index permutation is assumed to be identity, i.e., \(\sigma(i)=i\).}
    \label{fig:non-decreasing complexity}
\end{figure}

The entire proof is based on contradiction. In Figure~\ref{fig:non-decreasing complexity}, we provide an illustrative example. Note that the mapping from ground truth \(\mathbf{z}_t\) to optimal estimation \(\hat{\mathbf{z}}^{*}_t\) is a permutation and element-wise transformation, it does not include mixing, and hence \(\color{blue}{e_i}\) exists if and only if \(\color{red}{\hat{e}^{*}_i}\) exists. Therefore, $|\widehat{\mathcal{M}}^{*}| = |\mathcal{M}|$. The core of the proof requires us to demonstrate that $|\widehat{\mathcal{M}}| < |\mathcal{M}|$ cannot be true.


Suppose the transitions are weakly diverse lossy as defined in Def.~\ref{def:Diversely Lossy Transformation}, then for each edge $z_{t,i} \to z_{t+1,j}$ in the transition graph, there must be a region of $z_{t,i}$ such that only $z_{t+1,j}$ is influenced by $z_{t,i}$. Consequently, the corresponding $\hat{z}_{t+1,j}$ and $\hat{z}_{t+1,i}$ are not independent, since no mixing process can cancel the influence of $z_{t,i}$. Therefore, the edge $\hat{z}_{t,i} \to \hat{z}_{t+1,j}$ in the estimated graph must exist.

Note that without the weakly diverse lossy transition assumption, this argument may not hold. For example, if $\hat{z}_{t+1,j}$ can be expressed as a function that does not depend on $z_{t,i}$, then even though the edge $z_{t,i} \to z_{t+1,j}$ exists, the estimated edge $\hat{z}_{t,i} \to \hat{z}_{t+1,j}$ may not exist. This could occur if, after the transformation $\mathbf{h}$, the influences in different paths from $z_{t,i}$ to $\hat{z}_{t+1,j}$ cancel out with each other.

\paragraph{Necessity Example}
An example that violates the assumption is as follows:
\begin{align}
    z_{t+1,i} &= z_{t,i} + \epsilon_{t+1,i} \nonumber \\
    z_{t+1,j} &= z_{t,i} + \epsilon_{t+1,j} \nonumber \\
    \hat{z}_{i} &= z_{i} \nonumber \\
    \hat{z}_{j} &= z_{i} - z_{j} \nonumber
\end{align}
Here, the mapping from $\mathbf{z}$ to $\hat{\mathbf{z}}$ is invertible. Writing down the mapping from $\hat{\mathbf{z}}_t$ to $\hat{\mathbf{z}}_{t+1}$, particularly for $\hat{z}_{t+1, j}$, yields:
\begin{align}
    \hat{z}_{t+1, j} &= (z_{t,i} + \epsilon_{t+1,i}) - (z_{t,i} + \epsilon_{t+1,j}) \nonumber \\
    &= \epsilon_{t+1,i} - \epsilon_{t+1,j} \nonumber
\end{align}
Clearly, this is independent of $\hat{z}_{t,i}$. Hence, in this scenario, the edge on the estimated graph does not exist. This explains the necessity for the weakly diverse lossy transition assumption. Furthermore, it can be seen that violating the weakly diverse lossy transition assumption would require a very specific design, such as the transition in an additive noise case and the transition on $z$ being linear, which is usually not the case in real-world scenarios. Generally, this requires that the influences from different paths from $z_{t,i}$ to $\hat{z}_{t+1,j}$ cancel each other out, a condition that is very challenging to fulfill in practical settings.

\paragraph{Permutation Indexing}
One may also ask about the permutation of the index between $\mathbf{z}_t$ and $\hat{\mathbf{z}}_t$. Since the transformation $\mathbf{h}$ is invertible, the determinant of the Jacobian should be nonzero, implying the existence of a permutation $\sigma$ such that
\[
(i,\sigma(i)) \in \operatorname{supp}(\mathbf{J}_{\mathbf{h}}),\, \forall i \in [n].
\]
Otherwise, if there exists an $i$ such that $[\mathbf{J}_{\mathbf{h}}]_{i,\cdot} = \mathbf{0}$ or $[\mathbf{J}_{\mathbf{h}}]_{\cdot,i} = \mathbf{0}$, such a transformation cannot be invertible. We can utilize this permutation $\sigma$ to pair the dimensions in $\mathbf{z}_t$ and $\hat{\mathbf{z}}_t$.

Since each ground-truth edge is preserved in the estimated graph, by Proposition~\ref{prop: s1}, the inequality $|\widehat{\mathcal{M}}| < |\widehat{\mathcal{M}}^{*}|$ cannot hold true. Thus, the lemma is proved.
\end{proof}

Then, according to this lemma, the transition complexity $|\widehat{\mathcal{M}}_{\hat{u}}|$ of the learned $\hat{\mathbf{m}}_{\hat{u}}$ should be greater than or equal to $|\widehat{\mathcal{M}}^{*}_{\hat{u}}|$, which is the complexity when using an accurate estimation of $\hat{\mathbf{g}}^{*}$. This relationship can be expressed as follows:

\[
 | \widehat{\mathcal{M}}_{\hat{u}} | \geq  |\widehat{\mathcal{M}}^{*}_{\hat{u}}| .
\]

By Lemma~\ref{lemma: id of C under optimal g}, the expected complexity of the estimated model $\mathbb{E}_{\mathcal{D}} |\hat{\mathcal{M}}^{*}_{\hat{u}} |$ is strictly larger than that of the ground truth $\mathbb{E}_{\mathcal{D}} | \mathcal{M}_u |$. This implies the following inequality:
\begin{equation}
    \mathbb{E}_{\mathcal{D}} | \widehat{\mathcal{M}}_{\hat{u}} | \geq \mathbb{E}_{\mathcal{D}} |\widehat{\mathcal{M}}^{*}_{\hat{u}}| > \mathbb{E}_{\mathcal{D}} | \mathcal{M}_u | .
\end{equation}

However, Assumption~\eqref{as: Mechanism Sparsity} requires that the estimated complexity must be less than or equal to the ground-truth complexity, leading to a contradiction. This contradiction implies that the estimated $\hat{\mathcal{C}}$ must match the ground truth up to label swapping. Consequently, this supports the conclusion of Theorem~\ref{thm: identifiability of C}.
\end{proof}

\subsection{Proof of Corollary~\ref{cor: identifiability of domain variables under mechanism function variability}}
\label{ap:Proof of Corollary}
\begin{corollary}[Identifiability under Function Variability]
    Suppose the data \(\mathcal{D}\) is generated from the nonstationary data generation process described in \eqref{Eq:mixing function} and \eqref{Eq:transition function of z}. Assume the transitions are weakly diverse lossy (Def.~\ref{def:Diversely Lossy Transformation}), and the mechanism separability assumption~\ref{as: Mechanism Separability} along with the following assumptions hold:
    \begin{enumerate}[label=\roman*.,ref=\roman*, start=5]
    \item \underline{(Mechanism Function Variability)} Mechanism Functions are sufficiently different. There exists \(K \in \mathbb{N}\) such that for all $u\neq u'$, there exists \(k \leq K\), $\mathcal{M}_{u}^k \neq \mathcal{M}_{u'}^k$ i.e. there exists index $(i,j)$ such that \(\left[\mathcal{M}_{u}^k\right]_{i,j} \neq \left[\mathcal{M}_{u'}^k\right]_{i,j}\).
    \item \underline{(Higher Order Mechanism Sparsity)} The estimated transition complexity on dataset \(\mathcal{D}\) is no more than ground truth transition complexity,
    \small\begin{equation}
    \mathbb{E}_{\mathcal{D}} \sum_{k=1}^K | \widehat{\mathcal{M}}_{\hat{u}}^k | \leq \mathbb{E}_{\mathcal{D}}  \sum_{k=1}^K | \mathcal{M}_{u}^k |. 
    \end{equation}\normalsize
    \end{enumerate}
    Then the domain variables $u_t$ are identifiable up to label swapping (Def.~\ref{def: Identifiable Domain Variables}).
\end{corollary}
\begin{proof}



With a strategy similar to the proof of Theorem~\ref{thm: identifiability of C}, we aim to demonstrate that using an incorrect cluster assignment $\hat{\mathcal{C}}$ will result in $\sum_{k=1}^K |\mathcal{M}^k_{\hat{u}}|$ being higher than the ground truth, thereby still enforcing the correct $u_t$.

Differing slightly from the approach in Theorem~\ref{thm: identifiability of C}, in this setting, we will first demonstrate that under any arbitrary $\hat{\mathcal{C}}$ assignment, the estimated complexity is no lower than the complexity in the ground truth, i.e., $\sum_{k=1}^K|\widehat{\mathcal{M}}^k| \geq \sum_{k=1}^K|\mathcal{M}^k|$. 

First, we address the scenario where two different domains have the same transition graph but with different functions, as otherwise, the previous lemma~\ref{lemma: id of C under optimal g} still applies. In cases where the same transition causal graph exists but the functions differ, assumption~\ref{as: Mechanism Function Variability} indicates that there exists an integer $k$ such that $\mathcal{M}^k_{u} \neq \mathcal{M}^k_{u'}$, meaning the ground truth support matrices are different. However, due to incorrect clustering, the learned transition must cover both cases. To substantiate this claim, we need to first introduce an extension of the non-decreasing complexity lemma.

\begin{lemma}[Non-decreasing Complexity under Mechanism Function Variability]\label{lemma: Non-decreasing Complexity Under Mechanism Function Variability}
    Suppose there exists an invertible transformation $\mathbf{h}$ which maps the ground truth $\mathbf{z}_t$ to the estimated $\hat{\mathbf{z}}_t$, and it is neither a permutation nor a component-wise transformation. Then, the transition complexity on the estimated $\hat{\mathbf{z}}_t$ is not lower than that on the ground truth $\mathbf{z}_t$, i.e., 
    \[
    \sum_{k=1}^K|\widehat{\mathcal{M}}^k| \geq \sum_{k=1}^K|\mathcal{M}^k|.
    \]
\end{lemma}
\begin{proof}
    We can extend the notation of the edges $e$ to the higher-order case $e^k$ to represent the existence of a non-zero value for the $k$-th order partial derivative $\frac{\partial^k m_i}{\partial z_j^k}$. Under the weakly diverse lossy transition assumption, it is always possible to find a region where the influence in $e^k$ cannot be canceled in $\hat{e}^k$. In this region, the mapping from $z_i$ to $\hat{z}_i$ can be treated as a component-wise transformation, since the influence of $\mathbf{z}$ other than $z_i$ is zero due to the lossy transition assumption. It is important to note that there is also an indexing permutation issue between $z_i$ and $\hat{z}_i$; the same argument in the permutation indexing part of the proof of lemma~\ref{lemma: Non-decreasing Complexity Under Mechanism Function Variability} applies.

    Since $\mathcal{M}^k$ represents the support of the $k$-th order partial derivative, this implies that $[\mathcal{M}^k]_{i,j} =1$ implies $[\mathcal{M}^{k'}]_{i,j} =1$ for all $k' \leq k$. We aim to show that if for the transition behind edge $z_{t,j} \to z_{t+1,i}$, there exists a $K$ such that $[\mathcal{M}^k_u]_{i,j}$ are different for two domains, then one of them must be a polynomial with order $K-1$. For this domain, $[\mathcal{M}^k_u]_{i,j} = 1$ when $k = 1, 2, \dots, K-1$ and $[\mathcal{M}^k_u]_{i,j} = 0$ when $k \geq K$.

    To demonstrate that the non-decreasing complexity holds, we need to show that after an invertible transformation $h$ to obtain the estimated version, $[\widehat{\mathcal{M}}^k_u]_{i,j}$ cannot be zero for $k < K-1$, which can be shown with the following proposition.

    \begin{proposition}\label{prop:S3}
        Suppose $f$ is a polynomial of order $k$ with respect to $x$. Then, for any invertible smooth function $h$, the transformed function $\hat{f} \coloneq h^{-1} \circ f \circ h$ cannot be expressed by a polynomial of order $k'$, when $k' < k$.
    \end{proposition}
    \begin{proof}
        Let's prove it by contradiction. Suppose $\hat{f} \coloneq h^{-1} \circ f \circ h$ can be expressed as a polynomial of order $k' < k$. It follows that the function $\hat{f}(x) = C$ has $k'$ roots (repeated roots are allowed), since $h$ is invertible. Therefore, $h \circ \hat{f}(x) = h(C)$ also has the same number of $k'$ roots. By definition, $h \circ \hat{f} = f \circ h$, which means $f \circ h = h(C)$ has $k'$ roots. However, since $h$ is invertible, or equivalently it is monotonic, the equation $f \circ h = h(C)$ having $k'$ roots implies that $f(x) = C'$ has roots $k'$. Yet, since $f$ is a polynomial of order $k$, it must have $k$ roots, contradicting the fundamental theorem of algebra, which means that they cannot have the same number of roots. Hence, the proposition holds.
    \end{proof}

    The advantage of support matrix analysis is that, provided there exists at least one region where the support matrix is non-zero, the global version on the entire space will also be non-zero. Based on the definition of diverse lossy transition in Def.~\ref{def:Diversely Lossy Transformation}, it is always possible to identify such a region where for an edge \(z_{t,i} \to z_{t+1,j}\), the mapping from \(z_{t+1,j}\) to \(z_{t+1,j}\) can be treated as a component-wise relationship. This is because no other variables besides \(z_{t+1,j}\) change in conjunction with \(z_{t,i}\) to cancel the effect. Therefore, proposition~\ref{prop:S3} applies, and as a result, the complexity is nondecreasing. Thus, the lemma is proved.
\end{proof}

With this lemma, we have shown that for an arbitrary incorrect domain partition result, the induced ground-truth transition complexity is preserved after the invertible transformation $\mathbf{h}$. This partition effectively combines two regions, as illustrated in Figures~\ref{fig: wrong C} and \ref{fig: hardness mechanism function variability}. Consequently, the transition complexity has the following relationship:
\[
\widehat{\mathcal{M}}^k_{\hat{u}} = \mathcal{M}^k_{u} + \mathcal{M}^k_{u'}.
\]
Here, \(u\) and \(u'\) represent the ground-truth values of the domain variables, and \(\hat{u}\) denotes the estimated version, defined as the binary addition of the two ground truths.

By assumption~\ref{as: Mechanism Function Variability}, the two ground truth transitions' complexity $\mathcal{M}^k_{u}, \mathcal{M}^k_{u'}$ are different, then with the same arguments in the proof of the lemma~\ref{lemma: id of C under optimal g}, we can show that the expected transition complexity with wrong domain assignment over the whole dataset is strictly larger than the ground truth complexity with correct domain assignment. And it is easy to see that when the estimated latent variables are equal to the ground truth, $\hat{\mathbf{z}}_t = \mathbf{z}_t$ then the lower bound is achieved when the estimated domains are accurate. Note that this argument is not a sufficient condition to say that the estimated $\hat{\mathbf{z}}_t$ is exactly the ground truth $\mathbf{z}_t$ or an optimal estimation of it, since there can be other formats of mapping from $\mathbf{z}_t$ to $\hat{\mathbf{z}}_t$ that generate the same complexity. But this is sufficient to prove that by pushing the complexity to small, the domain variables $u_t$ must be recovered up to label swapping. This concludes the proof.
\end{proof}

\subsection{Proof of Theorem~\ref{thm: identifiability of z}}\label{ap:Proof of Identifiability of the Latent Causal Processes}
\begin{theorem}[Identifiability of the Latent Causal Processes]
    Suppose that the data \(\mathcal{D}\) is generated from the nonstationary data generation process~\eqref{Eq:mixing function}, \eqref{Eq:transition function of z}, which satisfies the conditions in Theorem~\ref{thm: identifiability of C} and Lemma~\ref{lemma:tdrl identifiability_of_z}, then the domain variables \(u_t\) are identifiable up to label swapping~(Def.~\ref{def: Identifiable Domain Variables}) and latent variables $\mathbf{z}_t$ are identifiable up to permutation and a component-wise transformation~(Def.~\ref{def: Identifiable Latent Causal Processes}).
\end{theorem}
\begin{proof}
From Theorem~\ref{thm: identifiability of C}, the domain variables \(u_t\) are identifiable up to label swapping, and then use the estimated domain variables in Lemma~\ref{lemma:tdrl identifiability_of_z}, the latent causal processes are also identifiable, that is, \(\mathbf{z}_t\) are identifiable up to permutation and a component-wise transformation.
\end{proof}

\subsection{Discussion on Assumptions}\label{ap: Discussion on Assumptions}
\subsubsection{Mechanism Separability}
Note that we assume that there exists a ground truth mapping \(\mathcal{C}: \mathcal{X} \times \mathcal{X} \to \mathcal{U}\), gives a domain index based on \(\mathbf{x}_{t-1}, \mathbf{x}_t\). The existence of such mapping means that the human can tell what the domain is based on two consecutive observations. If two observations are not sufficient, then it can be modified to have more observation steps as input, for example \(\mathbf{x}_{\leq t}\) or even full sequence \(\mathbf{x}_{[1:T]}\). If the input has future observation, which means that \(\mathbf{x}_{>t}\) is included, then this is only valid for sequence understanding tasks in which the entire sequence will be visible to the model when analyzing the time step \(t\). For prediction tasks or generation tasks, further assumptions on \(\mathcal{C}\) such as the input only contains \(\mathbf{x}_{<t}\) should be made, which will be another story. Those variants are based on specific application scenarios and not directly affect our theory, for brevity, let us assume the two-step case.

\subsubsection{Mechanism Sparsity}
This is a rather intuitive assumption in which we introduce some form of sparsity in the transitions, and our task is to ensure that the estimated transition maintains this sparsity pattern. This requirement is enforced by asserting an equal or lower transition complexity as defined in Assumption~\ref{as: Mechanism Sparsity}. Similar approaches, grounded in the same intuition, are also explored in the reinforcement learning setting, as discussed in works by Lachapelle et al.~\cite{lachapelle2024nonparametric} and Hwang et al.~\cite{hwang2024quantized}. The former emphasizes the identifiability result of the independent components, which necessitates additional assumptions. In contrast, the latter focuses on the RL scenario, requiring the direct observation of the latent variables involved in the dynamics, which leaves significant challenges in real-world sequence-understanding tasks, where the states are latent. And it is also extensively discussed in the nonlinear ICA literature~\cite{zheng2022on,zheng2023generalizing}, in which such a sparsity constraint was added to the mixing function.

\subsubsection{Mechanism Variability}
The assumption of mechanism variability requires that causal dynamics differ between domains, which requires at least one discrete edge variation within the causal transition graphs. This assumption is typically considered reasonable in practical contexts; humans identify distinctions between domains only when the differences are substantial, which often involves the introduction of a new mechanism or the elimination of an existing one. Specifically, this assumption requires a minimal alteration, a single edge change in the causal graph, to be considered satisfied. Consequently, as long as there are significant differences in the causal dynamics among domains, this criterion is fulfilled.

\subsubsection{Mechanism Function Variability}\label{ap:Discussion on the hardness of identifiability under mechanism function variability}
\begin{wrapfigure}{r}{5.5cm}
\vspace{-3.5em}
    \centering
    \includegraphics[width=0.35\textwidth]{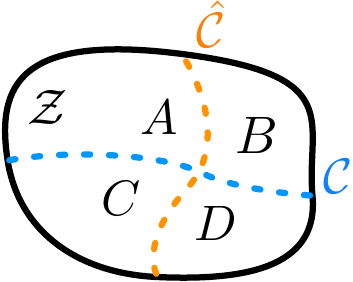}
    \caption{Correct domain separation \(\mathcal{C}\) and incorrect domain separation \(\hat{\mathcal{C}}\) of \(\mathbf{z}_{t+1}\), given a fixed \(\mathbf{z}_t\).}
    \label{fig: hardness mechanism function variability}
    \vspace{-2em}
\end{wrapfigure}
In this section, we will further discuss the mechanism function variability introduced in Corollary~\ref{cor: identifiability of domain variables under mechanism function variability}. One might question the necessity of this assumption. To illustrate this issue, we claim that if we only assume that the mechanism functions differ across domains but without this extended version of the variability assumption, i.e., for \(u \neq u'\), \(\mathbf{m}_u \neq \mathbf{m}_{u'}\), then under this proposed framework, the domain variables \(u_t\) are generally unidentifiable.

In Figure~\ref{fig: hardness mechanism function variability}, we present a simple example of the space of \(\mathbf{z}_{t+1}\) given a fixed \(\mathbf{z}_t\). For the sake of brevity, assume that there are two domains. By the mechanism separability assumption~\ref{as: Mechanism Separability}, the space \(\mathcal{Z}\) of \(\mathbf{z}_{t+1}\) is divided into two distinct parts, each corresponding to one domain. In this illustration, \(\mathcal{C}\) denotes the partition created by the ground truth transition function:
\begin{equation}
\mathbf{z}_{t+1} = 
\begin{cases} 
\mathbf{m}_1(\mathbf{z}_t,\epsilon) & \text{if } u_{t+1} = 1, \\
\mathbf{m}_2(\mathbf{z}_t,\epsilon) & \text{if } u_{t+1} = 2.
\end{cases}
\end{equation}

Then the question arises: when the domain assignment is incorrect, that is, \(\hat{\mathcal{C}} \neq \mathcal{C}\), can we still get the same observational distributions, or equivalently, can we obtain the same distribution for \(\mathbf{z}_{t+1}\)?

The answer is yes. For the ground truth transition, \(\mathbf{m}_1(\mathbf{z}_t, \epsilon) \in A \cup C\) and \(\mathbf{m}_2(\mathbf{z}_t, \epsilon) \in B \cup D\). In the case of an incorrect partition \(\hat{\mathcal{C}}\), it is sufficient to have \(\hat{\mathbf{m}}_1(\mathbf{z}_t, \epsilon) \in A \cup B\) and \(\hat{\mathbf{m}}_2(\mathbf{z}_t, \epsilon) \in C \cup D\). Ensuring that the conditional distribution \(p(\mathbf{z}_{t+1} \mid \mathbf{z}_t)\) is matched everywhere, we can create two different partitions on domains, yet still obtain exactly the same observations. That makes the domain variables \(u_t\) unidentifiable in the general case.
\paragraph{How does the previous mechanism variability assumption work?}
In the assumption of mechanism variability (Assumption~\ref{as: Mechanism Variability}), the support matrices of the Jacobian of transitions across different domains differ. Consider a scenario where the ground truth partition is \(\mathcal{C}\), denoted by \(A, C \mid B, D\). If an incorrect estimation occurs, where our estimated partition is \(\hat{\mathcal{C}}\), represented as \(A, B \mid C, D\), then the estimated transition in domain one should cover the transitions in both \(A\) and \(B\), and similarly for the second domain. This leads to an increase in complexity within the estimated Jacobian support matrix, as discussed in the previous sections. Consequently, this complexity forces the sets \(B\) and \(C\) to be empty, resulting in \(\hat{\mathcal{C}}\) converging to \(\mathcal{C}\).

\paragraph{How about mechanism function variability?}
Roughly speaking, and as demonstrated in our experiments, the mechanism variability assumption previously discussed is already sufficient to identify domain changes in both synthetic and real-world settings. This sufficiency arises because the assumption only requires a single differing spot, even though some transition functions behind some edges may persist across different domains. As long as there is one edge spot that can separate the two domains, this condition is met. In the relatively rare case where all edges in the causal dynamic transition graphs are identical across two different domains and only the underlying functions differ, we can still demonstrate identifiability in this scenario by examining differences in the support of the higher-order partial derivative matrices.

\subsubsection{Weakly Diverse Lossy Transition}
The weakly diverse lossy transition assumption requires that each variable in the latent space can potentially influence a set of subsequent latent variables, and such transformations are typically non-invertible. This implies that given the value of \(\mathbf{z}_{t+1}\), it is generally challenging to precisely recover the previous \(\mathbf{z}_t\); equivalently, this mapping is not injective. Although this assumption requires some explanation, it is actually considered mild in practice. Often in real-world scenarios, different current states may lead to identical future states, indicating a loss of information. The ``weakly diverse'' of this assumption suggests that the way information is lost varies between different dimensions, but there is some common part among them, hence the term ``weakly diverse''. In the visualization example shown in Figure~\ref{fig:vis}, we can clearly see this pattern, in which the scene is relatively simple and it is very likely that in two different frames, the configuration of the scene or the value of the latent variables are the same but their previous states are completely different.

\section{Experiment Settings}
\label{ap:experiment}
\subsection{Synthetic Dataset Generation}\label{ap:synthetic}
\begin{wraptable}{r}{5.8cm}
\vspace{-1em}
\centering
\caption{Synthetic Dataset Statistics}
\label{tab:syn_dataset_stats}
\begin{tabular}{cc}
\toprule
\textbf{Property} & \textbf{Value} \\ 
\midrule
Number of State & 5 \\ 
Dimension of \( \mathbf{z}_t \) & 8 \\ 
Dimension of \( \mathbf{x}_t \) & 8 \\ 
Number of Samples & 32,000 \\ 
Sequence Length & 15 \\ 
\bottomrule
\end{tabular}
\vspace{-1em}
\end{wraptable}
The synthetic dataset is constructed in accordance with the conditions outlined in Theorems~\ref{thm: identifiability of C} and \ref{thm: identifiability of z}. Transition and mixing functions are synthesized using multilayer perceptrons (MLPs) initialized with random weights. The mixing functions incorporate the LeakyReLU activation function to ensure invertibility. The dataset features five distinct values for the domain variables, with both the hidden variables \(\mathbf{z}_t\) and the observed variables \(\mathbf{x}_t\) set to eight dimensions. A total of 1,000 sequences of domain variables were generated. These sequences exhibit high nonstationarity across domains, which cannot be captured with a single Markov chain. This was achieved by initially generating two distinct Markov chains to generate two sequences of domain indices. Subsequently, these sequences were concatenated, along with another sequence sampled from a discrete uniform distribution over the set \(\{1, 2, 3, 4, 5\}\), representing the domain indices.


For each sequence of domain variables, we sampled a batch size of 32 sequences of hidden variables \(\mathbf{z}_t\) beginning from a randomly initialized initial state \(\mathbf{z}_0\). These sequences were generated using the randomly initialized multilayer perceptron (MLP) to model the transitions. Observations \(\mathbf{x}_t\) were subsequently generated from \(\mathbf{z}_t\) using the mixing function as specified in Eq.~\ref{Eq:mixing function}. Both the transition functions in the hidden space and the mixing functions were shared across the entire dataset. A summary of the statistics for this synthetic dataset is provided in Table~\ref{tab:syn_dataset_stats}. For detailed implementation of this data generation process, please refer to our accompanying code in Sec. \ref{sec:Reproducibility}.

\subsection{Real-world Dataset}\label{ap:real}
\paragraph{Hollywood Extended}~\cite{bojanowski2014weakly}
The Hollywood dataset contains 937 video clips with a total of 787,720 frames containing sequences of 16 different daily actions such as walking or sitting from 69 Hollywood movies. On average, each video comprises 5.9 segments, and 60.9\% of the frames are background.

\paragraph{CrossTask}~\cite{zhukov2019cross}
The CrossTask dataset features videos from 18 primary tasks. According to \cite{Lu_2021_ICCV}, we use the selected 14 cooking-related tasks, including 2552 videos with 80 action categories. On average, each video in this subset has 14.4 segments, with 74.8\% of the frames classified as background.

\subsection{Mean Correlation Coefficient}
MCC, a standard metric in the ICA literature, is utilized to evaluate the recovery of latent factors. This method initially computes the absolute values of the correlation coefficients between each ground truth factor and every estimated latent variable. Depending on the presence of component-wise invertible nonlinearities in the recovered factors, either Pearson’s correlation coefficients or Spearman’s rank correlation coefficients are employed. The optimal permutation of the factors is determined by solving a linear sum assignment problem on the resultant correlation matrix, which is executed in polynomial time.

\section{Implementation Details}
\subsection{Prior Likelihood Derivation}\label{ap:derive}
Let us start with an illustrative example of stationary latent causal processes consisting of two time-delayed latent variables, i.e., $\mathbf{z}_t = [z_{1,t}, z_{2,t}]$, i.e., $z_{i,t} = m_i(\mathbf{z}_{t-1}, \epsilon_{i,t})$ with mutually independent noises, where we omit the \(u_t\) since it is just an index to select the transition function \(m_i\). Let us write this latent process as a transformation map $\mathbf{m}$ (note that we overload the notation $m$ for transition functions and for the transformation map):

\begin{equation}
    \begin{bmatrix}
    z_{1,t-1} \\
    z_{2,t-1} \\
    z_{1,t} \\
    z_{2,t} \\
    \end{bmatrix} 
    =\mathbf{m} \left(
    \begin{bmatrix}
    z_{1,t-1} \\
    z_{2,t-1} \\
    \epsilon_{1,t} \\
    \epsilon_{2,t}
    \end{bmatrix}
    \right).
\end{equation}

By applying the change of variables formula to the map $\mathbf{m}$, we can evaluate the joint distribution of the latent variables $p(z_{1,t-1}, z_{2,t-1}, z_{1,t},z_{2,t})$ as:
\begin{equation}\label{eq:cvt-1}
p(z_{1,t-1}, z_{2,t-1}, z_{1,t},z_{2,t}) = p(z_{1,t-1}, z_{2,t-1}, \epsilon_{1,t},\epsilon_{2,t}) / \left|\det \mathbf{J}_\mathbf{m}\right|,
\end{equation}
where $\mathbf{J}_\mathbf{m}$ is the Jacobian matrix of the map $\mathbf{m}$, which is naturally a low-triangular matrix:
$$
\mathbf{J}_\mathbf{m} = 
\begin{bmatrix}
1 & 0 & 0 & 0\\
0 & 1 & 0 & 0\\
\frac{\partial z_{1,t}}{\partial z_{1,t-1}} & \frac{\partial z_{1,t}}{\partial z_{2,t-1}} & \frac{\partial z_{1,t}}{\partial \epsilon_{1,t}} & 0 \\ 
\frac{\partial z_{2,t}}{\partial z_{1,t-1}} & \frac{\partial z_{2,t}}{\partial z_{2,t-1}} & 0 & \frac{\partial z_{2,t}}{\partial \epsilon_{2,t}}
\end{bmatrix}.
$$
Given that this Jacobian is triangular, we can efficiently compute its determinant as $\prod_i \frac{\partial z_{i,t}}{\partial \epsilon_{i,t}}$. Furthermore, because the noise terms are mutually independent, and hence $\epsilon_{i,t} \perp \epsilon_{j,t}$ for $j \neq i$ and $\epsilon_t \perp \mathbf{z}_{t-1}$, we can write the RHS of Eq.~\ref{eq:cvt-1} as:

\begin{equation}\label{eq:example}
\begin{aligned}
    p(z_{1,t-1}, z_{2,t-1}, z_{1,t},z_{2,t}) &= p(z_{1,t-1}, z_{2,t-1}) \times p(\epsilon_{1,t},\epsilon_{2,t}) / \left|\det \mathbf{J}_\mathbf{m}\right| \quad (\text{because }\epsilon_t \perp \mathbf{z}_{t-1})\\
    &= p(z_{1,t-1}, z_{2,t-1}) \times \prod_i p(\epsilon_{i,t}) / \left|\det \mathbf{J}_\mathbf{m}\right| \quad (\text{because }\epsilon_{1,t} \perp \epsilon_{2,t})
\end{aligned}    
\end{equation}

Finally, by canceling out the marginals of the lagged latent variables $p(z_{1,t-1}, z_{2,t-1})$ on both sides, we can evaluate the transition prior likelihood as:

\begin{equation}\label{eq:example-likelihood}
p( z_{1,t},z_{2,t} \mid z_{1,t-1}, z_{2,t-1}) = \prod_i p(\epsilon_{i,t}) / \left|\det \mathbf{J}_\mathbf{m}\right| = \prod_i p(\epsilon_{i,t}) \times \left|\det \mathbf{J}_\mathbf{m}^{-1}\right|.
\end{equation}

Now we generalize this example and derive the prior likelihood below. 

Let $\{\hat{m}^{-1}_i\}_{i=1,2,3...}$ be a set of learned inverse transition functions that take the estimated latent causal variables, and output the noise terms, i.e., $\hat{\epsilon}_{i,t} = \hat{m}^{-1}_{i}\left(u_t, \hat{z}_{i,t}, \hat{\mathbf{z}}_{t-1} \right)$. 

Design transformation $\mathbf{A} \rightarrow \mathbf{B}$ with  low-triangular Jacobian as follows:
\small
\begin{align}
\underbrace{
\begin{bmatrix}
\hat{\mathbf{z}}_{t-1},
\hat{\mathbf{z}}_{t}
\end{bmatrix}^{\top}
}_{\mathbf{A}}
\textrm{~mapped to~}
\underbrace{
\begin{bmatrix}
\hat{\mathbf{z}}_{t-1},
\hat{\boldsymbol{\epsilon}}_{t}
\end{bmatrix}^{\top}
}_{\mathbf{B}},
~
with~
\mathbf{J}_{\mathbf{A} \rightarrow \mathbf{B}} = 
\begin{pmatrix}
\mathbb{I}_{n} & 0 \\
* & \text{diag}\left(\frac{\partial m^{-1}_{i,j}}{\partial \hat{z}_{jt}}\right)
\end{pmatrix}.
\end{align}
Similar to Eq.~\ref{eq:example-likelihood}, we can obtain the joint distribution of the estimated dynamics subspace as:
\vspace{-0.5mm}
\begin{align}
\log p(\mathbf{A}) = \underbrace{\log p\left(\hat{\mathbf{z}}_{t-1}\right) + \sum_{j=1}^n \log p(\hat{\epsilon}_{i,t})}_\text{Mutually independent noise} + \log \left(\lvert \det \left(\mathbf{J}_{\mathbf{A} \rightarrow \mathbf{B}}\right) \rvert \right) \label{eq:np-joint}.\\
\log p\left(\hat{\mathbf{z}}_t \mid \hat{\mathbf{z}}_{t-1}, u_t\right) = \sum_{i=1}^n \log p(\hat{\epsilon}_{i,t}\mid u_t)+ \sum_{i=1}^n \log \Big| \frac{\partial m^{-1}_{i}}{\partial \hat{z}_{i,t}}\Big|.
\end{align}

\subsection{Derivation of ELBO}\label{ap:ELBO}
Then the second part is to maximize the Evidence Lower BOund (\textsc{ELBO}) for the VAE framework, which can be written as:

\begin{equation}
\begin{split}
\textsc{ELBO}\triangleq &\log p_{\text{data}}(\{\mathbf{x}_t\}_{t=1}^{T}) - D_{KL}(q_{\phi}(\{\mathbf{z}_t\}_{t=1}^{T}\mid \{\mathbf{x}_t\}_{t=1}^{T})\mid\mid p_{\text{data}}(\{\mathbf{z}_t\}_{t=1}^{T}\mid \{\mathbf{x}_t\}_{t=1}^{T}))\\
=& \mathbb{E}_{\mathbf{z}_t} \log p_{\text{data}}(\{\mathbf{x}_t\}_{t=1}^{T}\mid\{\mathbf{z}_t\}_{t=1}^{T}) - D_{KL}(q_{\phi}(\{\mathbf{z}_t\}_{t=1}^{T}\mid \{\mathbf{x}_t\}_{t=1}^{T})\mid\mid p_{\text{data}}(\{\mathbf{z}_t\}_{t=1}^{T}\mid \{\mathbf{x}_t\}_{t=1}^{T}))\\
=&  \mathbb{E}_{\mathbf{z}_t}\log p_{\text{data}}(\{\mathbf{x}_t\}_{t=1}^{T}\mid\{\mathbf{z}_t\}_{t=1}^{T}) - \mathbb{E}_{\mathbf{z}_t} \left(\log q_{\phi}(\{\mathbf{z}_t\}_{t=1}^{T}\mid\{\mathbf{x}_t\}_{t=1}^{T}) - \log p_{\text{data}}(\{\mathbf{z}_t\}_{t=1}^{T}) \right]\\
=&\mathbb{E}_{\mathbf{z}_t} \left(\log p_{\text{data}}(\{\mathbf{x}_t\}_{t=1}^{T}\mid\{\mathbf{z}_t\}_{t=1}^{T}) + \log p_{\text{data}}(\{\mathbf{z}_t\}_{t=1}^{T}) - \log q_{\phi}(\{\mathbf{z}_t\}_{t=1}^{T}\mid\{\mathbf{x}_t\}_{t=1}^{T})  \right)\\
=& \mathbb{E}_{\mathbf{z}_t}\left(\underbrace{\sum_{t=1}^{T} \log p_{\text{data}}(\mathbf{x}_t\mid\mathbf{z}_t)}_{-\mathcal{L}_{\text{Recon}}} 
+ \underbrace{ \sum_{t=1}^{T}\log  p_{\text{data}}(\mathbf{z}_t\mid \mathbf{z}_{t-1},u_t) 
- \sum_{t=1}^{T}\log q_{\phi}(\mathbf{z}_t\mid\mathbf{x}_t)}_{-\mathcal{L}_{\text{KLD}}}\right)
\end{split}
\end{equation}

\subsection{Reproducibility}\label{sec:Reproducibility}
All experiments are performed on a GPU server with 128 CPU cores, 1TB memory, and one NVIDIA L40 GPU. For synthetic experiments, we run the baseline methods with implementation from \url{https://github.com/weirayao/leap} and \url{https://github.com/xiangchensong/nctrl}.
For real-world experiments, the implementation is based on \url{https://github.com/isee-laboratory/cvpr24_atba}.

Our code is also available via \url{https://github.com/xiangchensong/ctrlns}.

\subsection{Hyperparameter and Train Details}\label{sec:training details}
For synthetic experiments, the models were implemented in \texttt{PyTorch 2.2.2}. We trained the VAE network using the AdamW optimizer with a learning rate of \(5 \times 10^{-4}\) and a mini-batch size of 64. Each experiment was conducted using three different random seeds and we reported the mean performance along with the standard deviation averaged across these seeds. The coefficient for the \(L_2\) penalty term was set to \(1 \times 10^{-4}\), which yielded satisfactory performance in our experiments. We also tested \(L_1\) penalty or \(L_2\) penalty with larger coefficients, the setting we used in this paper provided the best stability and performance.
All other hyperparameters of the baseline methods follow their default values from their original implementation. For real-world experiments, we follow the same hyperparameter setting from the baseline ATBA method. In the Hollywood dataset, we used the default 10-fold dataset split setting and calculated the mean and standard derivation from those 10 runs. For the CrossTask dataset, we calculate the mean and standard derivation using five different random seeds.

\section{Visualization on Action Segmentation}\label{ap:vis}
We visualize some examples from the Hollywood dataset. As shown in Figure~\ref{fig:vis} we can see that our \method~can estimate the actions more accurate than baseline method.
\begin{figure}[ht]
    \centering
    \begin{subfigure}[b]{0.48\textwidth}
        \centering
        \includegraphics[width=\textwidth]{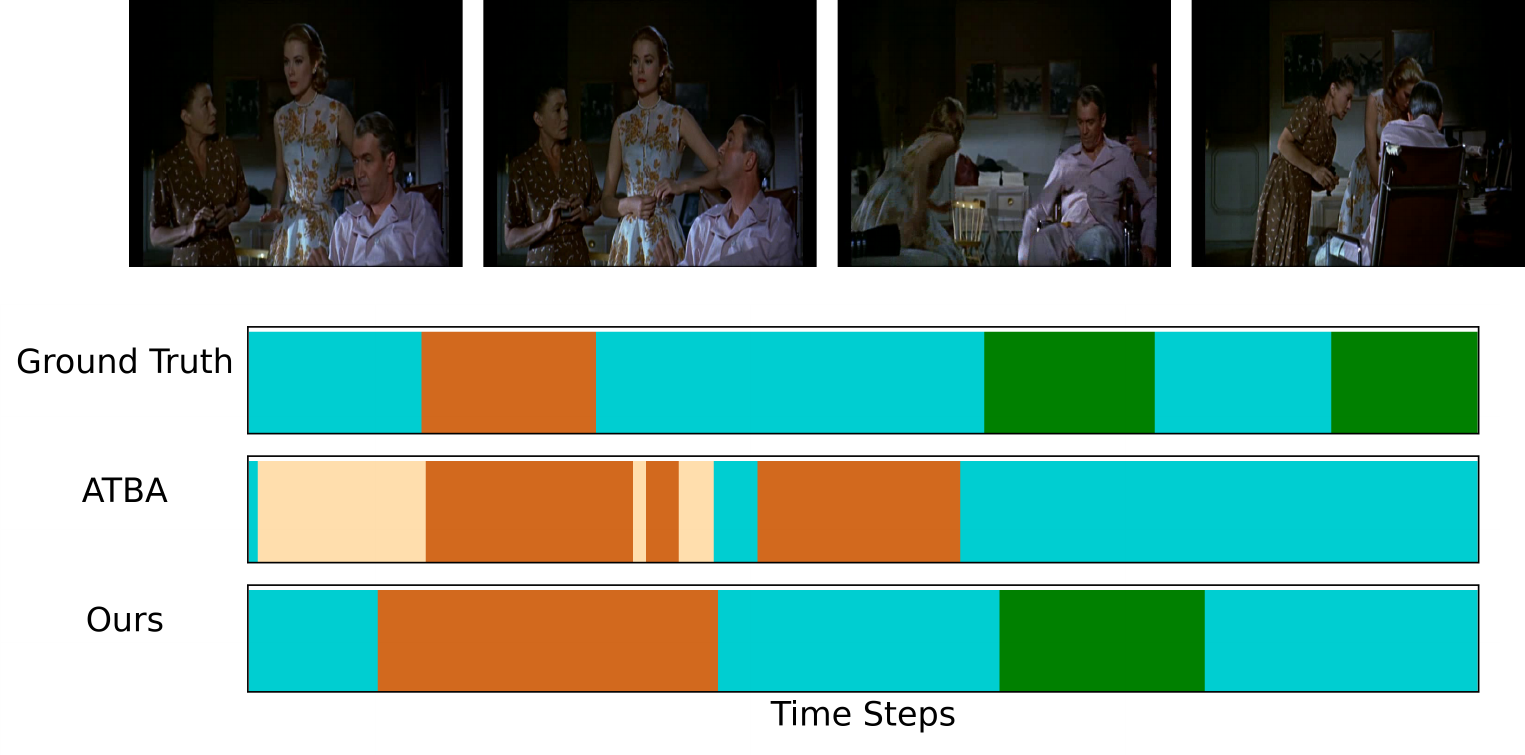}
        \label{fig:vis1}
    \end{subfigure}
    \hfill
    \begin{subfigure}[b]{0.48\textwidth}
        \centering
        \includegraphics[width=\textwidth]{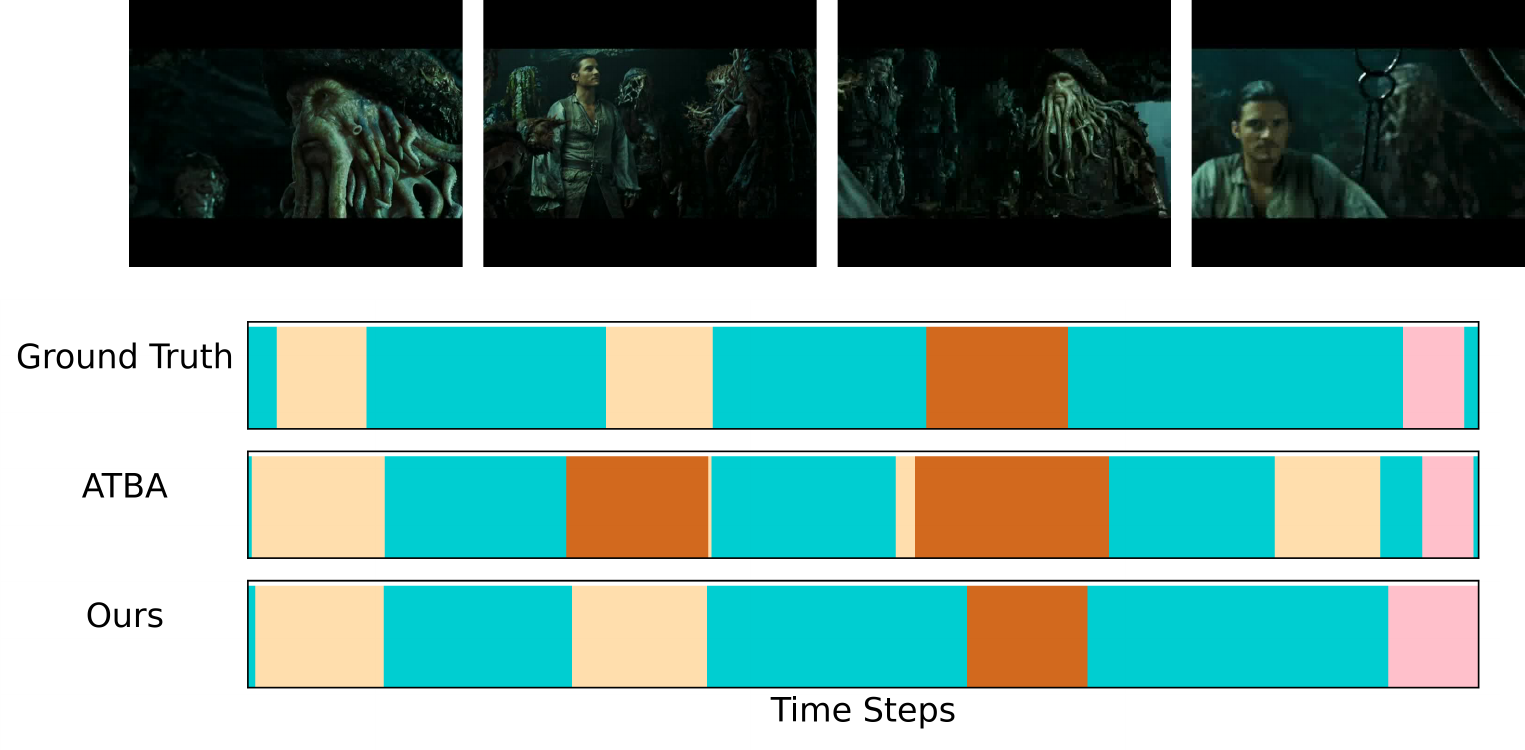}
        \label{fig:vis2}
    \end{subfigure}
    \caption{Two illustrative visualization on the action segmentation task in Hollywood dataset. The colors indicate the ground truth and the estimated action labels for each frame from baseline ATBA and our proposed \method.}
    \label{fig:vis}
\end{figure}

\section{Extended Related Work}\label{ap:related work}

\subsection{Causal Discovery with Latent Variables}
Various studies have focused on uncovering causally related latent variables. For example, \cite{silva2006learning,kummerfeld2016causal,huang2022latent} use vanishing Tetrad conditions~\cite{spearman1928pearson} or rank constraints to detect latent variables in linear-Gaussian models, whereas \cite{shimizu2009estimation,cai2019triad,xie2020generalized,xie2022identification} rely on non-Gaussianity in their analyses of linear, non-Gaussian models.
Additionally, some methods seek to identify structures beyond latent variables, leading to hierarchical structures.
Certain hierarchical model-based approaches assume tree-like configurations, as seen in \cite{pearl1988probabilistic,zhang2004hierarchical,choi2011learning,drton2017marginal}, while other methods consider a more general hierarchical structure~\cite{xie2022identification,huang2022latent}.
Nonetheless, these approaches are restricted to linear frameworks and encounter increasing difficulties with complex datasets, such as videos.

\subsection{Causal Temporal Representation Learning}
In the context of sequence or time series data, recent advances in nonlinear Independent Component Analysis (ICA) have leveraged temporal structures and nonstationarities to achieve identifiability. Time-contrastive learning (TCL)~\cite{hyvarinen2016unsupervised} exploits variability in variance across data segments under the assumption of independent sources. Permutation-based contrastive learning (PCL)~\cite{hyvarinen2017nonlinear} discriminates between true and permuted sources using contrastive loss, achieving identifiability under the uniformly dependent assumption. The i-VAE~\cite{khemakhem2020variational} uses Variational Autoencoders to approximate the joint distribution over observed and nonstationary regimes. Additionally, (i)-CITRIS~\cite{lippe2022citris, lippe2023causal} utilizes intervention target information to identify latent causal factors. Other approaches such as LEAP~\cite{yao2021learning} and TDRL~\cite{yao2022temporally} leverage nonstationarities from noise and transitions to establish identifiability. CaRiNG~\cite{chen2024caring} extended TDRL to handle non-invertible generation processes by assuming sequence-wise recoverability of the latent variables from observations. 

All the aforementioned methods either assume stationary fixed temporal causal relations or that the domain variables controlling the nonstationary transitions are observed. To address unknown or unobserved domain variables, HMNLICA~\cite{halva2020hidden} integrates nonlinear ICA with a hidden Markov model to automatically model nonstationarity. However, this method does not account for the autoregressive latent transitions between latent variables over time. IDEA~\cite{li2024how} combines HMNLICA and TDRL by categorizing the latent factors into domain-variant and domain-invariant groups. For the variant variables, IDEA adopts the same Markov chain model as HMNLICA, while for the invariant variables, it reduces the model to a stationary case handled by TDRL. Both iMSM~\cite{balsells2023identifiability} and NCTRL~\cite{song2023temporally} extend this Markov structure approach by incorporating transitions in the latent space but continue to assume that the domain variables follow a Markov chain.

\subsection{Weakly-supervised Action Segmentation}
Weakly-supervised action segmentation techniques focus on dividing a video into distinct action segments using training videos annotated solely by transcripts~\cite{bojanowski2014weakly, huang2016connectionist, richard2018neuralnetwork, ding2018weakly, li2019weakly, chang2019d3tw, richard2017weakly, souri2021fast, chang2021learning, lu2021weakly, kuehne2017weakly, souri2022fifa, kuehne2018hybrid, zhang2023hoi}. Although these methods have varying optimization objectives, many employ pseudo-segmentation for training by aligning video sequences with transcripts through techniques like Connectionist Temporal Classification (CTC)~\cite{huang2016connectionist}, Viterbi~\cite{richard2017weakly, richard2018neuralnetwork, li2019weakly, lu2021weakly, kuehne2017weakly, kuehne2018hybrid}, or Dynamic Time Warping (DTW)~\cite{chang2019d3tw, chang2021learning}. For instance, \cite{huang2016connectionist} extends CTC to consider visual similarities between frames while evaluating valid alignments between videos and transcripts. Drawing inspiration from speech recognition, \cite{kuehne2017weakly, richard2017weakly, kuehne2018hybrid} utilize the Hidden Markov Model (HMM) to link videos and actions. \cite{ding2018weakly} initially produces uniform segmentations and iteratively refines boundaries by inserting repeated actions into the transcript. \cite{richard2018neuralnetwork} introduces an alignment objective based on explicit context and length models, solvable via Viterbi, to generate pseudo labels for training a frame-wise classifier. Similarly, \cite{li2019weakly} and \cite{lu2021weakly} propose novel learning objectives but still rely on Viterbi for optimal pseudo segmentation. Both \cite{chang2019d3tw, chang2021learning} use DTW to align videos to both ground-truth and negative transcripts, emphasizing the contrast between them. However, except for \cite{ding2018weakly}, these methods require frame-by-frame calculations, making them inefficient. More recently, alignment-free methods have been introduced to enhance efficiency. \cite{souri2021fast} learns from the mutual consistency between frame-wise classification and category/length pairs of a segmentation. \cite{lu2022set} enforces the output order of actions to match the transcript order using a novel loss function. Although POC~\cite{lu2022set} is primarily set-supervised, it can be extended to transcript supervision, making its results relevant for comparison.

\section{Limitations}\label{sec:limitations}
As noted in Sec.~\ref{sec:remark}, our main theorem relies on the condition that causal graphs among different domains must be distinct. Although our experiments indicate that this assumption is generally sufficient, there are scenarios in which it may not hold, meaning that the transition causal graphs are identical for two different domains, but the actual transition functions are different. We have addressed this partially through an extension to the mechanism variability assumption to higher-order cases (Corollary~\ref{cor: identifiability of domain variables under mechanism function variability}). However, dealing with situations where transition graphs remain the same across all higher orders remains a challenge. We acknowledge this as a limitation and suggest it as an area for future exploration.

We also observed that the random initialization of the nonlinear ICA framework can influence the total number of epochs needed to achieve identifiability, as illustrated in Figure~\ref{fig:simulation three phases}. Also, for the computational efficiency, the TDRL framework we adopted involves a prior network that calculated each dimension in the latent space one by one, thus making the training efficiency suboptimal. Since this is not directly related to major claim which is our sparse transition design, we acknowledge this as a limitation and leave it for future work.

\section{Boarder Impacts}\label{ap:broader impacts}
This work proposes a theoretical analysis and technical methods to learn the causal representation from time-series data, which 
facilitate the construction of more transparent and interpretable models to understand the causal effect in the real world. 
This could be beneficial in a variety of sectors, including healthcare, finance, and technology.
In contrast, misinterpretations of causal relationships could also have significant negative implications in these fields, which must be carefully done to avoid unfair or biased predictions.

\end{document}